%% file: main.tex
\documentclass[10pt]{article} % For LaTeX2e
\usepackage[preprint]{tmlr}

\usepackage{hyperref}
\usepackage{url}
\usepackage{acro}                    % For acronym management
\usepackage[version=4]{mhchem}         % For chemical formulas
\usepackage{comment}                    % For commenting out sections
\usepackage{pdfpages}                   % For including PDF pages
\usepackage{xcolor}                     % For colors (enhanced version)
\usepackage{amsmath,amssymb,amsfonts}   % Enhanced math support
\usepackage{mathtools}
\usepackage{amsthm}
\newtheorem{theorem}{Theorem}
\usepackage{algorithmic}                % For algorithms
\usepackage{subcaption}
\usepackage{textcomp}                   % For text companion symbols
\usepackage{xspace}                     % For intelligent spacing
\usepackage{enumitem}                   % For enhanced lists
\usepackage{mathrsfs}                   % For curly script fonts
\usepackage{bm}                         % For bold math symbols
\usepackage{ifthen}                     % For conditional logic (needed for commenting system)
\usepackage{tikz}
\usetikzlibrary{positioning}
\usetikzlibrary{arrows.meta}
\usetikzlibrary{calc}

\usepackage{cleveref}

\usepackage{siunitx}
\DeclareUnicodeCharacter{2605}{\ensuremath{\bigstar}}
\DeclareUnicodeCharacter{2212}{\ensuremath{-}}

\newboolean{draft_comments}
\setboolean{draft_comments}{false}
\ifthenelse{\boolean{draft_comments}}
{
  \newcommand{\Jake}[1]{\textcolor{blue}{[JAKE: #1]}}
  \newcommand{\Michael}[1]{\textcolor{green}{[MICHAEL: #1]}}
  \newcommand{\Melanie}[1]{\textcolor{red}{[MELANIE: #1]}}
  \newcommand{\Matthew}[1]{\textcolor{magenta}{[MATTHEW: #1]}}
  \newcommand{\todo}[1]{\textcolor{orange}{[TODO: #1]}}
}
{
  \newcommand{\Jake}[1]{}
  \newcommand{\Michael}[1]{}
  \newcommand{\Melanie}[1]{}
  \newcommand{\Matthew}[1]{}
  \newcommand{\todo}[1]{}
}

\input{acronyms.tex}

\newcommand{\castle}{\ac{CaStLe}\xspace}
\newcommand{\mcastle}{\ac{M-CaStLe}\xspace}

\title{M-CaStLe: Uncovering Local Causal Structures in\\Multivariate Space-Time Gridded Data}

\author{\name J. Jake Nichol \email jefnich@sandia.gov \\
      \addr Scientific Machine Learning, Center for Computing Research\\
      Sandia National Laboratories
      \AND
      \name Michael Weylandt \email michael.weylandt@baruch.cuny.edu \\
      \addr Zicklin School of Business \\
      Baruch College, CUNY
      \AND
      \name G. Matthew Fricke \email mfricke@unm.edu\\
      \addr Department of Computer Science \\
      University of New Mexico
      \AND
      \name Jhayron Perez-Carrasquilla \email jhayron@umd.edu\\
      \addr Department of Atmospheric \& Oceanic Science \\
      University of Maryland
      \AND
      \name Melanie E. Moses \email melaniem@unm.edu\\
      \addr Department of Computer Science \\
      University of New Mexico\\
      \addr Santa Fe Institute}

\def\month{MM}  % Insert correct month for camera-ready version
\def\year{YYYY} % Insert correct year for camera-ready version
\def\openreview{\url{https://openreview.net/forum?id=XXXX}} % Insert correct link to OpenReview for camera-ready version

\begin{document}

\maketitle

\begin{abstract}
Causal graph discovery for space-time systems is challenging in high-dimensional gridded data, which often has many more grid cells than temporal observations per cell. The \castle meta-algorithm was developed to address that niche under space-time locality and stationarity assumptions, but it is currently limited to univariate analyses. In this work, we present \mcastle. \mcastle generalizes the local embedding and parent-identification phases of \castle to jointly model local within-variable and cross-variable space-time causal structures in gridded data. Like \castle, by constraining candidate parents to a constant-size space-time neighborhood and pooling spatial replicates, \mcastle increases effective sample size to make discovery tractable in high-dimensional settings. We further decompose the resulting multivariate stencil graph into reaction and spatial graphs to aid interpretation in complex settings.

We study \mcastle in four settings: a multivariate space-time \acl{VAR} benchmark with known ground truth, an \acl{ADR} \acl{PDE} verification problem with derived physical reference structure, an atmospheric chemistry case study in a low-temporal-sample regime, and an \acl{ENSO} study on reanalysis data, identifying phase-dependent ocean--atmosphere coupling. Across these settings, \mcastle more accurately recovers multivariate causal structure in controlled settings and identifies important physical dynamics in real-world case studies. Overall, \mcastle advances causal discovery for multivariate space-time systems while retaining interpretability at the grid level.
\end{abstract}

\acresetall

\section{Introduction}
Causal discovery estimates causal structure from observational data when controlled interventions are infeasible~\citep{Glymour.2019.10.3389/fgene.2019.00524}. This setting is common in large-scale space-time systems, where experiments may be unethical, impractical, or impossible~\citep{Runge.2023.10.1038/s43017-023-00431-y}. Causal discovery methods are now widely used across the sciences, including Earth, health, and social systems~\citep{Ebert-Uphoff.2012.10.1175/JCLI-D-11-00387.1,Cooper.2015.10.1093/jamia/ocv059,Runge.2019.10.1038/s41467-019-10105-3,Nowack.2020.10.1038/s41467-020-15195-y,Feder.2022.10.1162/tacl_a_00511,Zanga.2022.10.1016/j.ijar.2022.09.004,Sadeghi.2023.10.48550/arXiv.2312.17375}. Earth system science is a canonical example: with one Earth and limited opportunities for manipulation, causal discovery must operate directly on observational, often gridded, space-time datasets.

% Since the advent of Granger causality~\citep{Granger.1969.10.2307/1912791}, the Rubin causal model~\citep{Rubin.2019.10.1080/24709360.2019.1670513}, causal graphs~\citep{Pearl.2016}, and the \ac{PC} algorithm~\citep{Spirtes.1993.10.1007/978-1-4612-2748-9}, causal inference and causal discovery of observed data have developed into a rigorous mathematical framework. Today, causal discovery has become a rich literature with many algorithms and applications throughout the sciences~\citep{Glymour.2019.10.3389/fgene.2019.00524,Runge.2023.10.1038/s43017-023-00431-y}, including the health, Earth, and social sciences~\citep{Ebert-Uphoff.2012.10.1175/JCLI-D-11-00387.1,Cooper.2015.10.1093/jamia/ocv059,Runge.2019.10.1038/s41467-019-10105-3,Nowack.2020.10.1038/s41467-020-15195-y,Feder.2022.10.1162/tacl_a_00511,Zanga.2022.10.1016/j.ijar.2022.09.004,Sadeghi.2023.10.48550/arXiv.2312.17375}. Finally, causal representation learning is an exciting nascent field that merges the flexibility and predictive power of machine learning with causal discovery techniques~\citep{Scholkopf.2021.10.1109/jproc.2021.3058954}.

% This work presents \mcastle, a causal discovery approach for multivariate space-time systems with gridded data. 
Dense gridded datasets generally enable the analysis of continuous effects over space, since they are regular and complete throughout the grid. However, these systems come with dimensionality challenges. Frequently, the number of grid cells scales faster than the number of temporal samples per grid cell~\citep{Runge.2019.10.1038/s41467-019-10105-3}. Atmospheric data exemplify this challenge: many variables on hundreds of thousands of grid cells with orders of magnitude fewer temporal observations per cell. That imbalance is one aspect of the \emph{curse of dimensionality}~\citep{Bellman1957,Buhlmann.2011.10.1007/978-3-642-20192-9_6}, where high-dimensionality relative to sample size challenges conventional statistical methods. It renders many forms of inference, including causal discovery, unreliable without exploiting spatial marginalization, strong regularization, or additional structure (e.g., locality/stationarity).

\citet{Nichol.2025.10.1029/2024jh000546} developed \castle, which can efficiently identify local causal relationships of a given quantity in high-dimensional space-time systems where traditional approaches fail. \castle leverages spatial causal regularities to transform the causal discovery problem.
% It converts a high-dimensional space with many graph nodes and limited observations into an embedding with fewer nodes and more abundant observations.
However, many scientific questions in complex space-time systems require analysis of multiple quantities per grid cell, such as temperature and soil moisture in Earth system monitoring of drought conditions~\citep{Sun.2021.10.3390/rs13224638} or infection dynamics in epidemiological modeling using infection severity, duration of infection, and population age~\citep{Ganesan.2021.10.1038/s41598-021-86084-7, Paul.2021.10.1007/s40031-020-00517-x}. This work introduces \mcastle, which generalizes \castle from univariate to multivariate gridded space-time systems.

\subsection{Gap: Multivariate Grid-Level Causal Discovery}
While \castle discovers transport and propagation structures for a single gridded field, applying it variable-by-variable and inferring cross-variable effects afterward is inherently post-hoc, scales poorly with the number of variables, and can propagate errors into the inferred multivariate dynamics. This sequential approach also foregoes the benefits of joint estimation that are standard in time series causal discovery methods like \ac{PCMCI}~\citep{Runge.2019.10.1126/sciadv.aau4996}. Furthermore, learning space-time causal structures from each variable independently may miss cross-variable confounding, leading to space-time estimation errors and incorrect inference of the underlying physical process.

Joint estimation enables analysis of complex multi-step causal pathways that span both spatial transport and variable interactions. Consider atmospheric chemical processes, such as those from volcanoes or wildfires, leading to global climate impacts~\citep{Dutton.1992.10.1029/92gl02495,Labitzke.1992.10.1029/91gl02940,parker1996,Soden.2002.10.1126/science.296.5568.727}. Understanding coupled space-time dynamics, where species transport globally while simultaneously affecting others locally, requires joint estimation of both spatial propagation and cross-variable interactions. Similar multivariate space-time coupling occurs across computational fluid dynamics~\citep{Wimer.2023.10.1103/physrevfluids.8.110511}, spatiotemporal pharmacokinetics~\citep{Guarin.2021.10.1038/s41598-021-91612-6,Klingelhuber.2024.10.1038/s42255-024-01025-8}, and computational chemistry~\citep{Higham.2008.10.1137/060666457,Owen.2024.10.1115/1.4064494}. 

Generalizing \castle to this setting requires (i) a stencil representation that captures within- and cross-variable interactions while maintaining interpretability at scale, and (ii) an embedding construction that exposes cross-variable dependencies without losing the locality assumptions that make \castle computationally tractable.

\subsection{Contributions}
In this work, we present a multivariate generalization of the \castle meta-algorithm for grid-level causal discovery in multivariate space-time data. Building on \citet{Nichol.2025.10.1029/2024jh000546}, we develop and evaluate \mcastle with the following key advances:
\begin{itemize}[leftmargin=*, itemsep=2pt]
    \item \textbf{Multivariate \ac{LENS}}: Generalizes \castle's \ac{LENS} construction from univariate fields to $V$ variables per grid cell, pooling spatial replicates while preserving the local neighborhood organization needed for grid-level interpretation.
    \item \textbf{Multivariate \ac{PIP}}: Generalizes the \ac{PIP} to jointly identify the parents of all $V$ center variables in the local neighborhood, enabling simultaneous discovery of cross-variable and spatial causal effects.
    \item \textbf{Multivariate Causal Stencil Graph and Decomposition}: Introduces a multivariate causal stencil graph and proposes a decomposition into \emph{reaction} and \emph{spatial} graphs to separately summarize inter-variable coupling and transport structure.
    \item \textbf{Univariate Recovery as a Special Case}: Shows that the original univariate \castle formulation is recovered when each grid cell contains a single variable.
    \item \textbf{Evaluation Frameworks}: Develops multivariate space-time \ac{VAR} benchmarks for analytically known ground truth and an \ac{ADR} \ac{PDE} verification case with derived physical reference structure to assess recovery across increasing $V$ and diverse dynamical regimes.
    \item \textbf{Empirical Verification and Real-World Demonstration}: Verifies recovery of multivariate causal structure in synthetic benchmarks with known ground truth and in an \ac{ADR} verification case with derived physical reference structure, and demonstrates practical utility in two real-world case studies: Mt.\ Pinatubo sulfate chemistry and observed, phase-dependent \ac{ENSO} ocean--atmosphere coupling.
\end{itemize}

% Extending \castle to multivariate systems presents two key challenges beyond the issues already solved for univariate cases. First, multivariate stencil graphs must represent interactions both within and between variables while maintaining interpretability at scale. Second, the embedding construction must efficiently capture cross-variable dependencies without losing the locality assumptions that make \castle computationally tractable.

\subsection{Organization}
This paper is organized as follows. \Cref{sec:background} reviews related work and summarizes the \castle framework this work generalizes. \Cref{sec:methods} presents \mcastle, including the multivariate \ac{LENS}, the multivariate \ac{PIP}, the multivariate causal stencil graph, and its decomposition into spatial and reaction graphs. \Cref{sec:math_foundations} provides theoretical motivation and analyzes computational and sample complexity. \Cref{sec:benchmarks} evaluates \mcastle in two controlled studies: synthetic multivariate space-time \acp{VAR} (\Cref{sec:VAR}) and on an \ac{ADR} \ac{PDE} verification case with derived physical reference structure (\Cref{sec:adr_dynamics}). \Cref{sec:case_studies} presents two real-world applications: an atmospheric chemistry case study (\Cref{sec:atmospheric_chemistry}) and observed \ac{ENSO} dynamics (\Cref{sec:enso}). We conclude with discussion, limitations, and future directions in \Cref{sec:discussion}.

\section{Related Work}
\label{sec:background}
Causal discovery in space-time systems has long had a fundamental tension between grid-level analysis and computational tractability~\citep{Tibau.2022.10.1017/eds.2022.11}. \citet{EbertUphoff.2012.10.1029/2012gl053269} led the application of causal discovery to Earth systems at the grid level using the \ac{PC} algorithm~\cite{Spirtes.1991.10.1177/089443939100900106}. However, they faced severe dimensionality and computational constraints, and limited their analysis to coarse grids. Their work highlighted the immediate challenges of causal discovery on space-time systems: computational expense and aggregation effects on neighboring grid cells.

To accommodate these barriers, recent work uses spatial dimensionality reduction techniques, e.g. \ac{PCA}~\citep{Weylandt.2024.10.1175/JCLI-D-23-0267.1} and $\delta$-MAPS~\citep{Fountalis.2018.10.1007/s41109-018-0078-z}. \ac{PCMCI}~\citep{Runge.2019.10.1126/sciadv.aau4996} is usually applied using spatial weighted averages or \ac{PCA}-varimax~\citep{Runge.2015.10.1038/ncomms9502, Kretschmer.2016.10.1175/jcli-d-15-0654.1,Runge.2019.10.1038/s41467-019-10105-3, Capua.2019.10.5194/esd-11-17-2020, Capua.2020.10.5194/wcd-1-519-2020, Krich.2020.10.5194/bg-17-1033-2020}. In short, dimensionality reduction is used to capture time series signals over large regions that represent global-scale phenomena, such as the \ac{ENSO}. This procedure is demonstrably effective for identifying large-scale patterns such as climate teleconnections~\citep{Wallace.1981.10.1175/1520-0493109<0784:titghf>2.0.co;2,Tibau.2022.10.1017/eds.2022.11,Galytska.2022.10.1002/essoar.10512569.1}, but necessarily eliminates local grid-level interactions. While large-scale patterns are important aspects of study in complex systems, the nature of their emergence is also important to understand. Local interactions determine the location and magnitude of larger patterns and other mid-scale phenomena, making them unsuitable for studying transient, spatially-evolving phenomena, such as wildfires, volcanic plumes, weather fronts, and storms~\citep{Nichol.2025.10.1029/2024jh000546}.
% While effective for large-scale phenomena, dimensionality reduction fundamentally destroys local multivariate space-time processes such as transport and chemistry, making them unsuitable for studying transient, spatially-evolving phenomena, such as wildfires, volcanic plumes, weather fronts, and storms~\citep{Nichol.2025.10.1029/2024jh000546}.
% The space-time dynamics identified are global teleconnection networks~\citep{Wallace.1981.10.1175/1520-0493109<0784:titghf>2.0.co;2,Tibau.2022.10.1017/eds.2022.11,Galytska.2022.10.1002/essoar.10512569.1}. 

In contrast to using dimensionality reduction, two prior works have explicitly approached local-scale space-time causal discovery. \citet{Zhu.2016.10.1109/tbdata.2017.2723899} developed pg-Causality, which seeks local dependencies in the space-time propagation of air pollutants in urban settings. pg-Causality was applied to multivariate data. \citet{Sheth.2022.10.1109/bigdata55660.2022.10020845} developed STCD for hydrological systems and leveraged known spatial structure for causal discovery, but it was not applied to multivariate data. While both used spatial structure for local analysis, both are designed for sparse sensor networks rather than dense, continuous space-time grids. Unlike sensor networks with fixed monitoring locations, gridded Earth system data presents thousands of spatially-organized cells with consistent neighborhood structure, requiring fundamentally different approaches. The challenges of performing causal discovery on high-dimensional multivariate gridded space-time systems at the grid level remain unresolved.

\citet{Nichol.2025.10.1029/2024jh000546} introduced \castle as a scalable framework for causal discovery of local grid-level dynamics in high-dimensional spatiotemporal systems. Box~1 summarizes the method, and Appendix \ref{app:u_castle} provides a deeper walk-through. In synthetic \ac{VAR} benchmarks, \castle substantially outperformed standard causal discovery approaches, including \ac{PC}, \ac{PCMCI}, \ac{DYNOTEARS}, and Granger/FullCI, particularly in low-sample regimes where unstructured methods were effectively at chance. Crucially, both formal analysis and experiments showed that \castle benefits from larger spatial domains: by exploiting repeated local structure, it converts spatial extent into statistical power, reversing the usual curse of dimensionality. Experiments on Burgers' equation further showed that \castle can recover advective dynamics in nonlinear continuous \ac{PDE} systems, whereas dimensionality-reduction pipelines such as PCA-varimax+\ac{PCMCI} failed to recover the correct directional structure. In Earth system applications, \castle recovered causal stencils from aerosol optical depth fields generated by simulations of the 1991 Mt.\ Pinatubo eruption that were consistent with known stratospheric wind patterns in both the simplified \ac{HSW-V} model and the fully coupled \ac{E3SMv2-SPA} model; by contrast, direct application of \ac{PC} to the same gridded data produced dense, largely uninterpretable graphs. These results established \castle as an effective framework for recovering local causal dynamics from gridded spatiotemporal data, but only in the univariate setting. The present work presents a multivariate generalization of \castle, broadening the framework from univariate fields to joint causal discovery over multiple co-located variables.

\begin{figure}[ht]
    \centering
    \includegraphics[width=\textwidth]{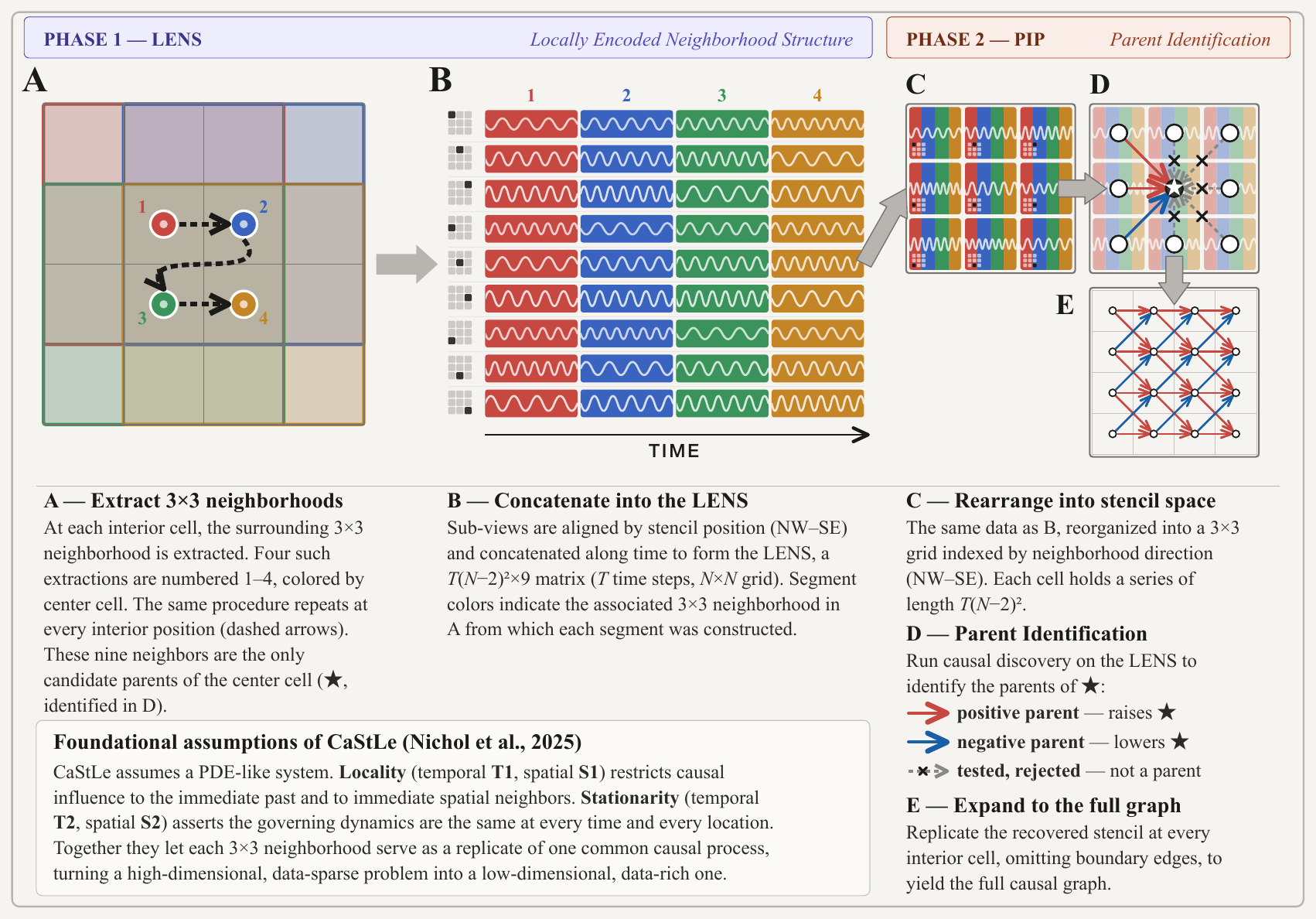}
    \caption{\citet{Nichol.2025.10.1029/2024jh000546}'s Univariate \castle at a Glance. See Appendix \ref{app:u_castle} for more details.}
    \label{fig:castle_box}
\end{figure}

\section{Proposed Method: \mcastle}
\label{sec:methods}
Multivariate \castle (\mcastle) is a multivariate generalization of \castle for discovering local space-time causal structures in gridded data. When each grid cell contains only a single variable, \mcastle reduces to the original univariate \castle framework. \mcastle produces the multivariate causal space-time stencil graph, which describes how a \emph{set of variables} interact within their Moore neighborhood over time. Because of this step in complexity, the multivariate stencil is often challenging to interpret immediately. To improve interpretability, we present the \textit{reaction graph} and the \textit{spatial graph}, which decompose the multivariate stencil output by \mcastle into a graph of inter-variable relationships (without a spatial aspect) and a graph of spatial relationships (without variable relationships).

\subsection{Mathematical Notation and Assumptions}
\mcastle can be applied to regular, fixed rectangular grids. For simplicity of exposition, we present the method and theory for a square $N\times N$ spatial grid, where $N$ denotes the grid dimension, with $T$ time samples and $V$ variables per grid cell. The extension to rectangular $N_x\times N_y$ grids is straightforward. The input data are represented by a tensor $\mathbf{X}\in\mathbb{R}^{N\times N\times V\times T}$. For a grid cell $r=(i,j)$, we denote by $\mathcal{N}(r)$ its Moore neighborhood, i.e., the set containing $r$ and its eight adjacent neighbors.

Building on the foundation that complex physical systems often exhibit consistent local dynamics, which develop into emergent to global behaviors, \mcastle applies to multivariate space-time domains characterized by locality and stationarity constraints. These assumptions enable an embedding that makes causal discovery tractable in high-dimensional gridded data, while preserving the essential causal structure. Below are the assumptions outlined by  \citet{Nichol.2025.10.1029/2024jh000546} for \castle, generalized for multivariate systems. For $V=1$, these assumptions reduce exactly to those used in the original univariate \castle formulation.

\begin{enumerate}[label={\bfseries T\arabic*)}]
    \item \emph{Temporal Locality}: direct causal parents occur only at lag $1$. Equivalently, for any $\tau \neq 1$, $X^{v}_{i,j,t-\tau} \not\to X^{w}_{k,l,t}$ for any variables $v,w$ and spatial coordinates $(i,j), (k,l)$.
    \item \emph{Temporal Causal Stationarity}: the dynamics governing evolution do not change over time. That is, $X^{v}_{i,j,t-1} \to X^{w}_{k,l,t} \Leftrightarrow X^{v}_{i,j,t-1+\tau} \to X^{w}_{k,l,t+\tau}$ for any time offset $\tau$.
\end{enumerate}
\begin{enumerate}[label={\bfseries S\arabic*)}]
    \item \emph{Spatial Locality}: if $(i,j)$ and $(k,l)$ are not neighbors, then $X^{v}_{i,j,t_1} \not\to X^{w}_{k,l,t_2}$ for any variables $v,w$ and times $t_1, t_2$.
    \item \emph{Spatial Causal Stationarity}: the dynamics do not change over space. That is, $X^{v}_{i,j,t-1} \to X^{w}_{k,l,t} \Leftrightarrow X^{v}_{i+\Delta_i,j+\Delta_j,t-1} \to X^{w}_{k+\Delta_i,l+\Delta_j,t}$ for any spatial offset $\Delta=(\Delta_i,\Delta_j)$ such that all shifted indices remain in the domain.
\end{enumerate}

These constraints limit interactions to neighboring cells within a one-step Moore neighborhood and adjacent time steps ($\tau = 1$), while ensuring uniform dynamics across space and time. Under these assumptions, multivariate systems can be represented as \acp{SCM} where each variable depends only on its Moore neighborhood across all variables at the previous time step (see \Cref{sec:math_foundations} for formal \ac{SCM} representation). Standard causal discovery assumptions (Markov condition, faithfulness) also apply~\citep{Spirtes.1993.10.1007/978-1-4612-2748-9}, though the locality constraints enable relaxation of the typically required causal sufficiency assumption~\citep{Nichol.2025.10.1029/2024jh000546}. These mathematical foundations are critical for \mcastle's computational tractability: the locality constraints transform an intractable $\mathcal{O}(TN^6V^3 2^{N^2V})$ complexity problem into a feasible $\mathcal{O}(TN^2V^3 2^{9V})$ approach (detailed complexity analysis in \Cref{subsec:complexity_analysis}).

\subsection{\mcastle}
\mcastle generalizes both phases of the \castle meta-algorithm to construct a multivariate \ac{LENS} and to recover both space-time and inter-variable dependencies within that representation. Input data consists of $V$ variables measured on an $N \times N$ grid over $T$ time steps, yielding a tensor $\mathbf{X} \in \mathbb{R}^{N \times N \times V \times T}$. \Cref{fig:methods_schematic} depicts each step of \mcastle. In this example, we illustrate a simple $4$$\times$$4$ original grid space, $G$, which has $V=3$ locally interacting variables, $a$, $b$, and $c$, with $T=500$ time samples.

\begin{figure*}[ht]
    \centering
    \includegraphics[width=\textwidth]{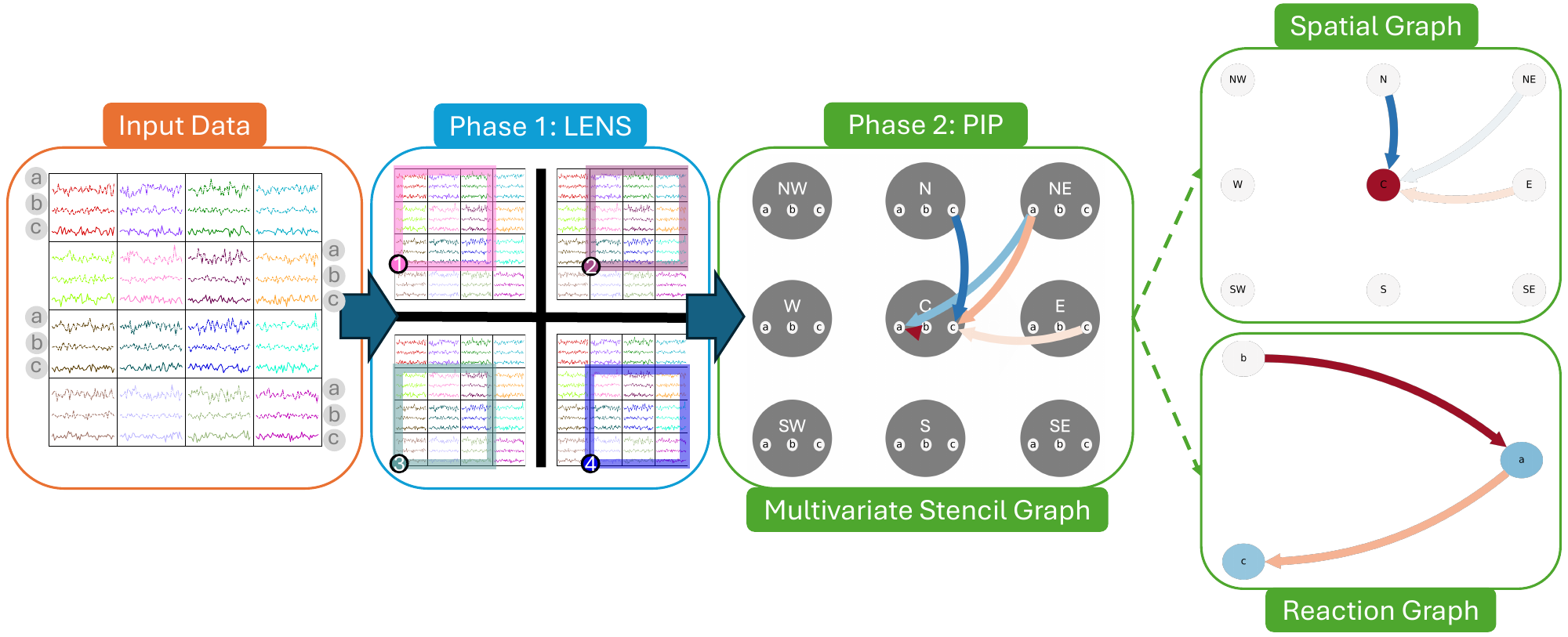}
    \caption{A schematic diagram of the input, computational phases, and output of \mcastle. Similar to \castle's procedure (c.f. Box~1), the first phase collects local neighborhood information into the \ac{LENS}, which now collects information for each variable's time series in each grid cell. The second phase applies the \ac{PIP} to every variable at every position in the \ac{LENS} to determine which variables cause the center variables from each location in the \ac{LENS}. Finally, the resulting multivariate stencil graph can be decomposed into the \emph{spatial graph} and \emph{reaction graph} for improved interpretability and potential analysis.}
    \label{fig:methods_schematic}
\end{figure*}

\subsubsection{Phase 1: \acf{LENS}}
Phase 1 builds the multivariate \ac{LENS} just as in the univariate case, except that at every position in the $3\times 3$ neighborhood it records a separate time series for each of the $V$ variables. The univariate \ac{LENS} is a $3$$\times$$3$ matrix where each element contains one time series per $N_\text{rep}$ interior spatial replicates, for a total length $L = T N_\text{rep}$. The multivariate \ac{LENS} generalizes the univariate version by including $V$ variables, forming a tensor in $\mathbb{R}^{3 \times 3 \times V \times L}$. In \Cref{fig:methods_schematic}, Phase 1 depicts the process the \ac{LENS} construction follows to collect time series from each Moore neighborhood as the window slides across $G$. It collects all three variables from each grid cell within the neighborhood window and concatenates them to the \ac{LENS}, according to their position relative to the center of the neighborhood window and the respective variable in each position. Like the univariate case, the multivariate \ac{LENS} retains all data without marginalization, thus being fully invertible.

\subsubsection{Phase 2: \acf{PIP}}
In univariate \castle, the \ac{PIP} adapts a given time series causal discovery algorithm to seek the parents of only the center cell in the \ac{LENS}. To generalize this step, the parent-identification problem is expanded from a single center variable to all $V$ center variables jointly. Rather than allowing a single target variable in the center position, we now allow all $V$ center variables at that position to serve as targets. Formally, let the center-position target set be
\[
    \mathcal{C} = \{X_{0,0,v,t} : v \in \{1,\dots,V\}\},
\]
where stencil coordinates are indexed relative to the center cell. The candidate parent set is
\[
    \mathcal{P} = \{X_{\delta_i,\delta_j,u,t-1} : \delta_i,\delta_j\in\{-1,0,1\},\ u\in\{1,\dots,V\}\},
\]
which contains $9V$ lag-1 candidate parents. The multivariate \ac{PIP} seeks, for each center variable $X_{0,0,v,t}\in\mathcal{C}$, the subset of $\mathcal{P}$ that is its direct parent set. Thus, any lag-1 variable in any stencil position is a potential direct parent of any center variable. The result is a multivariate stencil graph, such as the one depicted in the third panel of \Cref{fig:methods_schematic}. This example illustrates a stencil of three variables with dependencies between them over space and time.

\subsubsection{Decomposition}
\label{subsubsec:decomposition}
While the stencil in \Cref{fig:methods_schematic} may be interpretable after careful viewing, multivariate stencils of more variables or with more dependencies can be challenging to parse visually. For that reason, we have developed a decomposition scheme to separately analyze the variable interactions and the spatial structure of all variables. The far right of \Cref{fig:methods_schematic} illustrates the spatial graph and reaction graph corresponding to the stencil on their left.

Computing the stencil decomposition is straightforward and similar for both the spatial and reaction graphs. To compute the spatial graph, the stencil links are aggregated along the variable dimension, and the location from which they originate is preserved. For example, in \Cref{fig:methods_schematic}, two links are coming from the NE position to the center, a negative dependence (light blue) via $a \rightarrow a$ and a positive dependence (orange) via $a \rightarrow c$, and both of those are aggregated to find one weakly negative link $\text{NE} \rightarrow \text{C}$ in the spatial graph. Note that there is a $b_{\mathrm{center}} \rightarrow a_{\mathrm{center}}$ link in the stencil and that it is represented as an autodependence link in the spatial graph, illustrated by the center graph node's coloring. The node and link colors associate with continuous link dependence strengths (omitted for simplicity) that is output by \mcastle's \ac{PIP}.

The reaction graph is computed by aggregating stencil links along the spatial dimension while preserving the variable dimension. This results in a graph of variables that represents the aggregate strengths of dependencies from any direction. For example, in \Cref{fig:methods_schematic}, there are two $c \rightarrow c$ links, one in the N location and one in the E location, strongly negative (blue) from the N and weakly positive (red) from the E. These are aggregated to form the light-blue $c$ node in the reaction graph.

To aggregate the stencil link coefficients, we use Fisher's z-transformation. It stabilizes the variance of the correlation coefficients, making them more suitable for averaging. The process involves converting each coefficient into a z-score, computing the arithmetic mean of the z-scores, and then converting the average z-score back to a correlation coefficient using the inverse Fisher's z-transformation. This method ensures that the combined value accurately reflects the underlying dependencies between variables.

\section{Theory and Complexity Analysis}
\label{sec:math_foundations}
This section establishes the mathematical foundations for the \mcastle framework, which enables tractable causal discovery in high-dimensional space-time gridded data. We formalize the structural causal model assumptions underlying our approach and demonstrate how locality constraints fundamentally transform the computational complexity from intractable to feasible scales. The theoretical framework presented here clarifies the assumptions underlying our causal inference approach and its computational advantages over standard multivariate methods.

Under the \mcastle assumptions, multivariate systems can be represented through \acp{SCM} of the form:
\begin{equation}
X^{v}_{i,j,t} = f^{v}_{i,j}(\mathbf{X}^{1:V}_{\mathcal{N}(i,j),t-1}, \eta^{v}_{i,j,t})
\end{equation}
where $\mathbf{X}^{1:V}_{\mathcal{N}(i,j),t-1}$ represents all $V$ variables in the neighborhood of location $(i,j)$ at the previous time step, and $\eta^{v}_{i,j,t}$ captures stochastic innovations. Under spatial causal stationarity, $f^{v}_{i,j} = f^{v}$ for all locations. This formulation strictly generalizes the univariate theory of \citet{Nichol.2025.10.1029/2024jh000546}, which is recovered when $V=1$.

\subsection{Effective Sample Size}
A major strength of \mcastle is that it leverages spatial replicates to achieve an effective sample size of $\Theta(T N^2)$ for an $N \times N$ grid with $T$ temporal samples per location.

This quadratic scaling in $N$ is the key to recovering causal structure even when $T$ is small. For example, in the Mt.\ Pinatubo case study, each location has only $T = 7$ temporal samples, but pooling $\Theta(N^2)$ spatial replicates yields orders of magnitude more effective samples. 

We first consider the idealized case where spatial replicates are independent, which yields the scaling results below. We then relax this assumption to account for spatial dependence. 

\begin{theorem}[Sample complexity with independent spatial replicates]
\label{thm:sample-complexity}
Assume Temporal Locality (T1), Temporal Causal Stationarity (T2), Spatial Locality (S1), and Spatial Causal Stationarity (S2), with Markov and faithfulness conditions, independent finite‐variance errors, and \emph{independent} spatial replicates. 

In \ac{PIP}, restrict candidate parents of the center variables to the $3\times 3$ Moore neighborhood across all $V$ variables at lag~1. Let $m(\varepsilon,\delta,9V)$ denote the sample requirement for the time‐series causal discovery algorithm used in \ac{PIP} at feature dimension $p=9V$. Then each local parent‐selection problem in \ac{PIP} has effective sample size
\begin{equation}
\label{eq:sample_scaling}
L = T\,N_\text{rep}
\end{equation}
where $T$ is the number of temporal samples per grid cell and $N\times N$ is the spatial grid dimension, and $N_\text{rep}=(N-2)^2$ are the number of interior grid cells. In particular, if
\begin{equation}
T \ge \frac{m(\varepsilon,\delta,9V)}{N_\text{rep}}
\end{equation}
then \mcastle identifies each edge correctly with probability at least $1 - \delta$, up to an error of $\varepsilon$.
\end{theorem}

\begin{proof}
(i) By T1 and S1, any parent of a center variable at time $t$ must lie in its Moore neighborhood at time $t-1$, so the candidate set has $9V$ variables. 
(ii) By T2 and S2, the joint distribution of $(X_{\text{neigh},t-1}, X_{\text{center},t})$ is identical for each interior window position, making the $(N-2)^2$ windows exchangeable replicates of the same local mechanism. 
(iii) Under the independence assumption, sliding the window over $T$ time steps yields $L = TN_\text{rep}$ i.i.d.\ replicates in the LENS without marginalisation. 
(iv) Since the time‐series causal discovery algorithm used in PIP’s sample bound $m(\varepsilon,\delta,9V)$ applies to these replicates, the stated inequality on $T$ ensures the error and confidence guarantees.
\end{proof}

\subsubsection{Replicate dependence} If correlations between replicates fall with distance, $k$ , the effective sample size is reduced only by a constant factor $1/C$, where $C$ is the \emph{design effect}, also known as the variance inflation factor (VIF). Following  \citet{bretherton1999effective}, let $\rho_k$ denote the average correlation between spatial replicates separated by $k$ grid steps (stencil units) and $N_\text{rep}$ is the number of independent spatial replicates, then the effective number of independent samples is,
\begin{equation}
\label{eq:vif_1}
N_{\mathrm{eff}} = \frac{N_\text{rep}}{\mathrm{VIF}}
\end{equation}
where
\begin{equation}
\label{eq:vif_2}
\mathrm{VIF} = 1 + 2 \sum_{k=1}^{N-1} \rho_k \left( 1 - \frac{k}{N_\text{rep}} \right)
\end{equation}
and $\rho_k$ is the lag-$k$ autocorrelation. For spatial replicates, this can be expressed as
\begin{equation}
\label{eq:de_1}
\mathrm{DE} = 1 + 2 \sum_{k=1}^\infty \rho_k
\end{equation}
where $\rho_k$ is the average correlation at lag $k$ (in stencil units). The effective sample size is then
\begin{equation}
\label{eq:de_2}
L_{\mathrm{eff}} = \frac{L}{\mathrm{DE}}
\end{equation}
so dependence multiplies error by $\Theta(\sqrt{\mathrm{DE}})$ while leaving the $\Theta(N^2)$ scaling intact.

If dependence {\it does not} fall with $k$, for example $\rho_k \sim k^{-\alpha}$ with $0 < \alpha < 1$, then $\mathrm{DE} = \Theta(N^{1-\alpha})$ and $L_{\mathrm{eff}} = \Theta(N^{1+\alpha})$, which is less than quadratic. In the extreme case where all values are the same in a row or column then, $\mathrm{DE} = \Theta(N)$ and $L_{\mathrm{eff}} = \Theta(N)$, eliminating the quadratic gain from spatial replication entirely.

The resulting effective sample size is $L_{\mathrm{eff}} = L / \mathrm{DE}$, which reduces the bound on $T$ by a constant factor when correlations between spatial replicates decay sufficiently fast. 
If correlations decay slowly (e.g.\ $\rho_k \sim k^{-\alpha}$ with $0 < \alpha < 1$), $L_{\mathrm{eff}}$ scales sub-quadratically with $N$ as discussed below.

\subsubsection{Dependence from sliding windows.}
The proof of \Cref{thm:sample-complexity} given above treats the $L = TN_\text{rep}$ window replicates as independent. 
In reality, overlapping $3\times 3$ Moore windows share lagged cells, inducing short-range spatial dependence. 
If $\rho_{i,j}$ is the mean correlation between windows shifted $i$ stencil units horizontally and $j$ units vertically, the design effect $\mathrm{DE}$ from \Cref{eq:de_2} generalizes to two dimensions as
\begin{equation}
\mathrm{DE} = 1 + 2\sum_{\substack{i,j \ge 0 \\ (i,j) \neq (0,0)}} \rho_{i,j} \left(1 - \frac{i}{N_x}\right) \left(1 - \frac{j}{N_y}\right),
\end{equation}
where $N_x=N_y=N-2$ are the counts of window positions in each dimension.

For a $3\times 3$ stencil, the fractional overlaps are:
\begin{align}
\rho_{1,0} &= \rho_{0,1} = \frac{6}{9} \quad 
\rho_{1,1} = \frac{4}{9}\\
\rho_{2,0} &= \rho_{0,2} = \frac{3}{9} \quad 
\rho_{2,1} = \rho_{1,2} = \frac{2}{9}\\
\rho_{2,2} &= \frac{1}{9}  \quad 
\rho_{i,j} = 0 \quad \forall i \ge 3 \text{ or } j \ge 3
\end{align}
Substituting and summing gives
$\mathrm{DE} \approx 1.019$
so the effective sample size is

$L_{\mathrm{eff}} = \frac{L}{\mathrm{DE}} \approx 0.981 L$
Thus the penalty from short-range dependence is negligible: over 98\% of the $\Theta(N^2)$ scaling is retained.

\subsubsection{Relevance to Mt.\ Pinatubo.}
As discussed in \Cref{sec:atmospheric_chemistry}, for the Mt.\ Pinatubo dataset $N=30$ and $T=7$, the independent-replicate bound gives  
$L = T(N-2)^2 = 7 \cdot 28^2 = \SI{5488}{}$. 
Applying the exact $3\times 3$ overlap adjustment,
\[
L_{\mathrm{eff}} \approx 0.981 \times 5488 \approx \num{5384}.
\]
This is $\approx 769\times$ greater than the temporal depth at a single location, corresponding to a $\sqrt{5384/7} \approx 27.7\times$ reduction in per-edge statistical error. 
This underlies the rationale for the approach and accounts for the strong performance observed in our experiments.

\subsection{Computational Complexity Analysis}
\label{subsec:complexity_analysis}
Generalizing from the complexity analysis of \citet[Appendix B]{Nichol.2025.10.1029/2024jh000546}, we analyze the computational complexity of \mcastle compared to standard multivariate causal discovery approaches. Following \citet{Kalisch:2007}, the worst-case complexity of constraint-based causal discovery is $\mathcal{O}(np^32^p)$ for $n$ samples and $p$ variables. For multivariate gridded data with $N \times N$ spatial grid, $V$ variables per cell, and $T$ time samples, standard methods yield $p = N^2V$ total variables and $n = T$ samples, giving complexity
\begin{equation}
\mathcal{O}(T \cdot (N^2V)^3 \cdot 2^{N^2V}) = \mathcal{O}(TN^6V^3 2^{N^2V})
\end{equation}

The multivariate \ac{LENS} transformation creates $T(N-2)^2$ samples with $9V$ variables (Moore neighborhood $\times V$ variables). Crucially, we only seek parents for the $V$ center variables, yielding
\begin{equation}
\mathcal{O}(T(N-2)^2 \cdot (9V)^2 \cdot V \cdot 2^{9V}) = \mathcal{O}(TN^2V^3 2^{9V})
\end{equation}
where the $V^3$ factor decomposes as $(9V)^2$ for conditional independence test complexity and $V$ target variables requiring parent identification.

The improvement is dramatic in both polynomial and exponential terms: $N^2$ vs $N^6$ (quadratic vs sixth-order scaling) and $2^{9V}$ vs $2^{N^2V}$ (constant vs grid-dependent exponent). For a modest $30 \times 30$ grid with $V=3$ variables, naive approaches require evaluating $\approx 2^{2700}$ parent sets compared to \mcastle's $2^{27}$, representing a reduction of over $10^{800}$ in search space size. This fundamental advantage stems from \mcastle's locality constraint: each variable has exactly $9V$ potential parents regardless of grid size, compared to $N^2V$ parents in naive approaches.

While the above comparison analysis assumes exhaustive parent set enumeration, practical algorithms employ search optimizations (greedy selection, constraint propagation). However, \mcastle maintains proportional advantages since both approaches benefit equally from such optimizations. The key insight is that locality fundamentally constrains the search space independent of the specific search strategy.

\section{Simulation Studies: Linear and Nonlinear Scenarios}
\label{sec:benchmarks}
We first evaluate \mcastle in two controlled settings: a synthetic \ac{VAR} benchmark, where recovery can be measured exactly against known ground truth, and an \acf{ADR} \ac{PDE} verification problem, where physically meaningful reference structure is derived from the governing dynamics. The former uses randomly generated \acp{VAR} in which the modeling assumptions are exactly satisfied, while the latter examines a more realistic continuous system spanning a range of advection and diffusion settings.

\subsection{Spatial \ac{VAR} Benchmark}
\label{sec:VAR}
% We developed random and stable multivariate space-time systems with two spatial dimensions. They mathematically map to ground truth causal stencil graphs for evaluating the performance of \mcastle with a variety of system parameters. With these, we can compare \mcastle results with exact solutions.

Our methodology for generating benchmark data builds upon the work used by \citet{Nichol.2025.10.1029/2024jh000546}, which is more completely detailed by \citet{Nichol.2023.10.2172/1991387}. They developed a procedure for generating benchmark datasets of stable 2D space-time systems through the systematic construction of coefficient matrices parameterizing \acp{VAR} of order 1. Causal graphs have a direct mapping from \acp{VAR}~\citep{Peters.2017,Runge.2019.10.1126/sciadv.aau4996}, which enables precise benchmark comparisons between \ac{VAR} modeled data and causal discovery estimated graphs.

The space-time \ac{VAR} methodology initializes a $3$$\times$$3$ matrix defining local grid-level dynamics between neighbors, called the \ac{NDM}. Random \acp{NDM} of predetermined sparsity, $d$, are generated to describe how every grid cell in the space is dependent on the grid cells in its Moore neighborhood. To simulate an entire grid, the \ac{NDM} can be structurally mapped to an $\bm{A}$ matrix for the entire grid. For an $N$$\times$$M$ grid space, $\bm{A} \in \mathbb{R}^{NM \times NM}$. Finally, most 2D \acp{VAR} are not numerically stable. To ensure stability, $\rho(\bm{A}) < 1.0$, where $\rho(\bm{A})$ is the spectral radius of $\bm{A}$~\citep[p.307]{Strang}. Through the \ac{NDM} definition, \acp{VAR} can simulate locality in physical systems.

% \subsubsection{Experimental Setup: Multivariate Space-Time \acp{VAR}}
% \label{subsec:mv_vars}
To generate benchmark datasets, we generalized \citet{Nichol.2023.10.2172/1991387}'s space-time \ac{VAR} procedure to multivariate systems by defining local grid-level dynamics between variables. Stable systems were generated using an accept-reject scheme, ensuring predefined sparsity and minimum signal strengths. This approach enabled the generation of data across a range of grid sizes, variable counts, densities, and coefficients. See Appendix \ref{app:data_generation} for full details of the procedure, and \Cref{subsubsec:data_gen} and Appendix \ref{app:data_generation_parameters} for parameter ranges.

\paragraph{Cartesian-\castle Baseline}
To isolate the benefit of \mcastle's joint multivariate local estimation, we introduce Cartesian-\castle as an ablation baseline. Cartesian-\castle decomposes the multivariate problem into separate within-variable and cross-variable analyses. First, univariate \castle is run independently on each variable's gridded field, producing one univariate causal stencil per variable that captures within-variable space-time structure. Second, each variable is spatially aggregated to a single time series via its spatial mean, and a standard non-spatial multivariate causal discovery algorithm is applied to these aggregated series to estimate inter-variable causal links. These outputs are then assembled into a multivariate stencil graph by placing the univariate stencils on the block diagonal and inserting inter-variable links only between the center nodes of the corresponding variable blocks. In this way, Cartesian-\castle preserves within-variable spatial structure and cross-variable dependence, but it does not jointly estimate them in a shared multivariate neighborhood and cannot represent cross-variable spatial parents away from the center location. It therefore provides a stronger baseline than direct non-spatial application alone while isolating the contribution of \mcastle's multivariate LENS and joint \ac{PIP}.

% \subsubsection{Metrics}
% We evaluated \mcastle's performance using precision, recall, and F$_1$ score, which are standard metrics for binary classification tasks. In the context of causal graph reconstruction, precision measures the proportion of inferred edges that are correct, recall measures the proportion of true edges that are successfully recovered, and F$_1$ score is the harmonic mean of precision and recall, providing a balanced measure of overall performance. Formal definitions are provided in Appendix \ref{app:metrics}.

\subsubsection{Data Generation}
\label{subsubsec:data_gen}
We used \ac{VAR}s to simulate space-time structural causal models on a fixed $4\times4$ grid over $T=1000$ time points. The parameter ranges were: number of variables: $V \in \{1,2,3,4,5,6\}$; link density: $d = E/(9V^2)\in(0,1]$ for $E$ graph edges; coefficient magnitude: $\gamma\in[0.1,1.0]$. To guarantee bounded dynamics, each candidate network was subjected to a 48-hour stability accept–reject scheme: any realization that blew up or failed our spectral‐radius criterion was discarded and re-sampled until exactly 30 stable replicates were obtained. Appendix \ref{app:data_generation_parameters} lists, for each $V$, the maximum number of links $E$ (and corresponding density $d$) for which all 30 replicates passed stability. For example, with $V=1$ we were able to generate fully dense stencils ($d=1$), whereas for $V=6$ the upper limit was approximately $d=0.19$. In total, 31,140 independent simulations were generated.

For the synthetic \ac{VAR} benchmark only, we impose toroidal (wrap-around) boundary conditions so that every cell has the same local neighborhood structure. The model is a spatially-coupled VAR(1). Using the data tensor
$\mathbf{X}\in\mathbb{R}^{N\times N\times V\times T}$, the update rule for variable $v$ at cell $(i,j)$ and time $t$ is
\begin{equation}
  X_{i,j,v,t}
  \;=\;
  \sum_{u=1}^{V}
  \sum_{(i',j')\,\in\,\mathcal{N}(i,j)}
    \beta_{vu}^{(i-i',\,j-j')}\,
    X_{i',j',u,\,t-1}
  \;+\;
  \varepsilon_{i,j,v,t},
  \label{eq:var}
\end{equation}
where
\begin{itemize}[leftmargin=*, itemsep=2pt]
  \item $\mathcal{N}(i,j)$ is the Moore neighborhood of cell $(i,j)$, comprising the cell itself and its 8 immediate neighbors (9 cells total), with toroidal (wrap-around) indexing so the grid has no boundary.
  \item $\beta_{vu}^{(i-i',\,j-j')}\in\mathbb{R}$ is the stencil coefficient governing the causal effect of variable $u$ at neighbor cell $(i',j')$ on variable $v$ at the center cell $(i,j)$. The proportion of non-zero coefficients is controlled by the link density $d = E/(9V^{2})$, where $E$ is the number of active causal links in the stencil.
  \item $\varepsilon_{i,j,v,t}\overset{\mathrm{i.i.d.}}{\sim}\mathcal{N}(0,\sigma^{2})$ is additive Gaussian noise with $\sigma=0.1$.
\end{itemize}
Stacking all $VN^{2}$ state values into a vector $\mathbf{x}^{(t)}\in\mathbb{R}^{VN^{2}}$,
equation~\eqref{eq:var} is equivalent to the global VAR(1)
\begin{equation}
  \mathbf{x}^{(t)} = \mathbf{A}\,\mathbf{x}^{(t-1)} + \boldsymbol{\varepsilon}^{(t)},
  \qquad \boldsymbol{\varepsilon}^{(t)}\sim\mathcal{N}(\mathbf{0},\sigma^{2}\mathbf{I}),
  \label{eq:global_var}
\end{equation}
where $\mathbf{A}\in\mathbb{R}^{VN^{2}\times VN^{2}}$ is the global dynamics matrix obtained by tiling the stencil coefficients across the toroidal grid. Stability requires the spectral radius $\rho(\mathbf{A})<1$, enforced by the accept--reject scheme described above.

\subsubsection{Benchmark Results}
\label{sec:var_results}
We present empirical results of \mcastle's performance on our \ac{VAR} benchmarks varying: the number of variables, the number of graph dependencies (graph links), and the magnitude of coefficients. \citet{Nichol.2025.10.1029/2024jh000546} found that \castle performed better as grid sizes increased from $4\times4$ to $10\times10$ (traditional methods performed worse). Our tests demonstrated the same for \mcastle and we chose to focus strictly on $4\times4$ grids for these results. This simplifies the analysis and represents the most challenging setting for \mcastle while being the easiest for traditional methods.

We implemented \mcastle's \ac{PIP} with three common time series causal discovery algorithms: time-series-adapted \ac{PC}~\citep{Spirtes.1991.10.1177/089443939100900106,tigramite}, \ac{PCMCI}~\citep{Runge.2019.10.1126/sciadv.aau4996}, and \ac{DYNOTEARS}~\citep{Pamfil.2020}. We compare these \mcastle{}d variants against two alternatives: direct application of the same causal discovery algorithms without a stencil embedding, and Cartesian-\castle, an ablation baseline that combines independently estimated univariate \castle stencils with non-spatial inter-variable causal discovery on spatially aggregated data. The empirical results that follow demonstrate that \mcastle successfully generalizes \castle to multivariate gridded systems and that its benefits persist across multiple \ac{PIP} implementations.

\paragraph{Multivariate Performance}
\Cref{fig:precision_recall_f1_castle_vs_cart_with_baselines} shows mean F$_1$ score, precision, and recall as the number of variables increases from one to six. The figure includes all generated graph link densities and coefficient magnitudes, neither of which had noticeable impacts on performance alone. Across all three causal discovery algorithm \acp{PIP}, \mcastle performs best overall. It maintains high precision while achieving substantially higher recall than both Cartesian-\castle and the direct non-spatial baselines, yielding the strongest F$_1$ scores throughout the full range of variable counts. Cartesian-\castle performs intermediately: it improves over direct baseline application by retaining within-variable spatial structure, but remains consistently below \mcastle, especially as the number of interacting variables increases. This indicates that exploiting locality alone is not sufficient; jointly estimating within-variable and cross-variable relationships in a shared multivariate stencil is critical for accurate recovery. The same qualitative ordering holds across \ac{PC}, \ac{PCMCI}, and \ac{DYNOTEARS}, indicating that the gains from \mcastle are not specific to any single \ac{PIP} implementation.

\begin{figure*}[ht!]
    \centering
    \includegraphics[width=\textwidth]{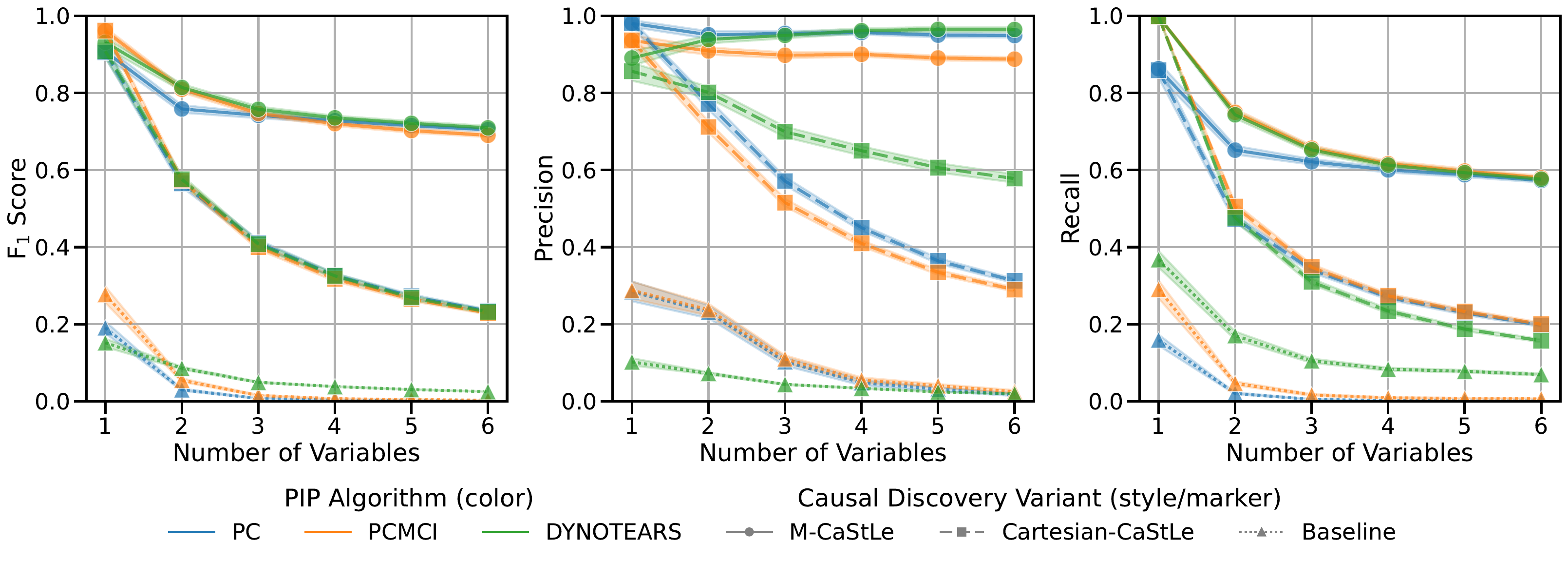}
    \caption{Results of the \ac{VAR} benchmark analysis comparing \mcastle, Cartesian-\castle, and direct non-stencil baselines as the number of variables increases from one to six. The three panels show mean F$_1$ score (left), precision (center), and recall (right) for \ac{PC}, \ac{PCMCI}, and \ac{DYNOTEARS} with 95\% confidence intervals. Across all three \acp{PIP}, \mcastle achieves the strongest overall performance, maintaining high precision and substantially higher recall than the alternatives. Cartesian-\castle improves over the direct baselines, but remains consistently below \mcastle, indicating that jointly estimating the multivariate local stencil structure provides additional benefit over independently estimating univariate spatial stencils and separately inferring non-spatial inter-variable links.}    \label{fig:precision_recall_f1_castle_vs_cart_with_baselines}
\end{figure*}

\paragraph{Impact of Graph Density on Recall}  
% To better understand \mcastle's declining performance as the number of variables increase, we investigated its precision and recall. The investigation is detailed fully in Appendix \ref{app:exploring_recall}. We found that precision is consistently $> 0.90$ for each \ac{PIP} but recall declines at a similar rate as F$_1$. This means that when \mcastle detects a graph link, it is very likely to be a true positive; but \mcastle is not detecting as many links as it should as $V$ increases. We conducted experiments on benchmark systems designed to isolate the effects of signal strength and variable count. The results, detailed in Appendix \ref{app:exploring_recall}, demonstrate that recall increases sharply with coefficient magnitude, achieving perfect recall when signals are sufficiently strong, even with up to 200 variables. These findings indicate that \mcastle can achieve high recall in large systems, provided the signal strength is adequately large, which requires lower relative graph density. Thus, declining performance with more variables is more likely explained by the limited signal to noise ratio of these increasingly complex systems due to our stability constraints.

Performance steadily declines as system complexity increases, from very high (F$_1 > 0.90$) to moderate scores for six variable systems. To understand why, we measured its precision and recall separately. As shown in \Cref{fig:precision_recall_f1_castle_vs_cart_with_baselines}, precision remains high across all three \ac{PIP} implementations, but recall, and thus F$_1$, declines with $V$. This means \mcastle finds few false positives, but misses an increasing number of true links in larger systems. In controlled experiments, isolating signal strength and variable count (Appendix~\ref{app:additional_var_results}), recall rises sharply with coefficient magnitude, reaching $100\%$ even at very large $V$ when signals are strong enough. This shows that, given sufficient signal strength (i.e.\ lower relative graph density), \mcastle can maintain high recall in large systems. Thus, its declining performance is driven primarily by the reduced signal-to-noise ratio imposed by our stability constraints.

This reflects a fundamental limit on causal detectability: enforcing $\rho(A) < 1$ on a dense transition matrix requires interaction strengths of $\mathcal{O}(1/\sqrt{n})$~\citep{MAY.1972.10.1038/238413a0,Geman.1986.10.1214/aop/1176992372}, where $n$ is the dimension of $A$, which vanish faster than the noise floor of any finite-sample estimator as model complexity grows, rendering true causal coefficients asymptotically indistinguishable from zero~\citep{Tsiamis.2021.10.48550/arxiv.2104.01120,Tosh.2025.10.1090/noti3044}.

\subsection{Advection-Diffusion-Reaction Dynamics}
\label{sec:adr_dynamics}
While the VAR benchmark is useful for understanding how \mcastle performs in controllable, idealized settings and for fairly comparing methodologies, it cannot capture more realistic physical dynamics. These dynamics are continuous and often include complex local relationships involving transport and diffusion. To examine \mcastle's performance in more realistic settings, we implemented an \ac{ADR} \ac{PDE} model that describes the transport and interaction of two chemical species.

\subsubsection{Model Definition}
The two species interacting in the model are denoted as \( u_1 \) and \( u_2 \). Their dynamics are defined by:
\begin{align}
\frac{\partial u_1}{\partial t} + \mathbf{v} \cdot \nabla u_1 &= D_1 \nabla^2 u_1 + R_1(u_1, u_2) \\
\frac{\partial u_2}{\partial t} + \mathbf{v} \cdot \nabla u_2 &= D_2 \nabla^2 u_2 + R_2(u_1, u_2)
\end{align}
where \( \mathbf{v} = (v_x, v_y) = (v \cos(\theta), v \sin(\theta)) \).

The reaction terms \( R_1(u_1, u_2) \) and \( R_2(u_1, u_2) \) govern the interaction between the two chemical species. Specifically, species \( u_1 \) undergoes a linear decay, converting into species \( u_2 \) at a rate controlled by two parameters: the decay rate (\( \alpha \)) and a conversion factor (\( \beta \)) that determines the proportion of \( u_1 \) converted into \( u_2 \). The reaction terms are defined as:
\[
R_1(u_1, u_2) = -\alpha u_1, \quad R_2(u_1, u_2) = \beta \alpha u_1
\]

For additional details on the model parameters, boundary conditions, and experimental setup, see Appendix \ref{app:ADR_details}.

\subsubsection{Ground Truth, Estimation, and Metrics}
As \citet{Rubenstein.2018.10.48550/arxiv.1608.08028} demonstrate, there is not a direct one-to-one mapping between causal graphs and differential equation systems, because the same underlying dynamics can be represented by multiple valid causal structures depending on variable choice, time scales, and levels of abstraction. That fundamental ambiguity poses a significant challenge for causal discovery algorithms applied to \ac{PDE}-governed systems, since there is no explicitly defined ground truth causal graph inherent to the differential equation specification.

Thus, to assess the quality of causal stencil graphs estimated by \mcastle, we derive physically meaningful reference structure from properties of the \ac{PDE} definition and a means of comparing the stencils to the reference. The two primary dynamical elements captured by the causal stencil graph are spatial dependence and interactions between species. Spatial dynamics can be modeled in the \ac{ADR} model by defining advection angle and velocity in \( \mathbf{v} \), and species interactions can be modeled in the reaction terms, \( R_1 \), and \( R_2 \).

\begin{figure*}[ht!]
    \centering
    \includegraphics[width=\textwidth]{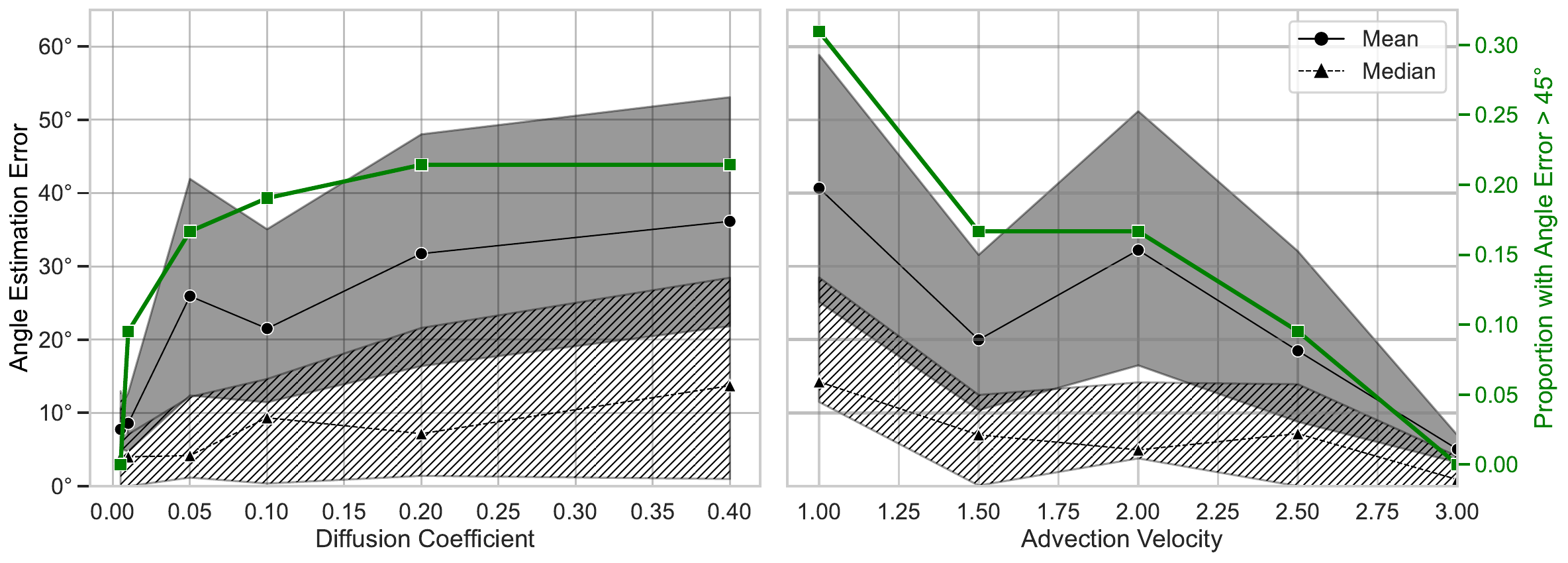}
    % \caption{Angle estimation error as a function of (left) diffusion coefficient and (right) advection velocity. Black circles connected by solid lines show mean angle estimation errors with a $95\%$ confidence interval, while black triangles connected by dotted lines show median angle estimation errors. Green squares show the proportion of cases where errors exceed \( 45^\circ \). For diffusion coefficients, mean errors remain relatively stable around \( 20^\circ \) before increasing at higher values, while median errors stay consistently lower around \( 5-15^\circ \). The proportion of large errors increases steadily from near \( 0\% \) to approximately \( 20\% \). For advection velocities, mean errors decrease sharply from around \( 40^\circ \) at \( v = 1.0 \) to near \( 0^\circ \) at \( v = 3.0 \), before rising again at \( v = 4.0 \), while median errors remain consistently low across all velocities. The proportion of large errors similarly decreases from approximately \( 30\% \) at low velocities to near \( 0\% \) at intermediate velocities, with a slight increase at the highest velocity. These results indicate optimal angle estimation performance occurs at low diffusion coefficients and intermediate advection velocities.}
    \caption{Angle estimation performance of \mcastle on the advection–diffusion–reaction \ac{PDE} verification experiments. \textbf{Left:} Mean (solid black line with circles) and median (dashed black line with triangles) of the absolute angular error $\Delta\theta$ as a function of diffusion coefficient, each with 95\% bootstrap confidence bands; green squares show the proportion of trials with $\Delta\theta>45^\circ$. \textbf{Right:} Same statistics plotted versus advection velocity. Results indicate robust transport-direction recovery at low diffusion and intermediate velocities, with increased error and outlier rates at high diffusion and extreme velocity.}
    \label{fig:angle_error}
\end{figure*}

To compare the discovered spatial dynamics to the \ac{ADR} advection angle, we derived an angle from the computed multivariate stencil graph. The approach involves computing a weighted circular mean of the angles represented by each edge, with weights determined by the correlation coefficients associated with each edge. We described the process mathematically in Appendix \ref{app:ADR_details}. To evaluate species interaction properties from the causal stencil graph output by \mcastle, we follow the decomposition procedure outlined in \Cref{subsubsec:decomposition} to obtain spatial and reaction graphs. The reaction graph can be directly compared to a causal graph representing the species interaction dynamics. For a detailed explanation of the estimation procedure, see Appendix \ref{app:ADR_details}.

To compute the angle estimation error, we find the minimum angle difference between the true and estimated angles~\citep[Appendix D]{Nichol.2025.10.1029/2024jh000546}. It is described as:
\[
\Delta \theta = \min\left( \left| \theta_1 - \theta_2 \right|, 360 - \left| \theta_1 - \theta_2 \right| \right)
\]
To measure the error of the reaction dynamics, we used the \( F_1 \) Score, defined in Appendix \ref{app:ADR_details}.

\paragraph{Experimental Parameters}
To evaluate the performance of \mcastle on realistic physical dynamics, we conducted a series of \ac{ADR} model experiments. Diffusion coefficients for the two species (\( D_1 \) and \( D_2 \)) varied together between \( 0.005 \) and \( 0.4 \), advection velocities (\( v \)) ranged from \( 1.0 \) to \( 3.0 \), with angles every $15^\circ$ from \( 0^\circ \) to \( 90^\circ \). Reaction rates (\( \alpha \)) were set to \( 1.0 \) with a scaling factor (\( \beta \)) of \( 1.0 \), which capture the decay of \( u_1 \) into \( u_2 \). For additional experimental details, see Appendix \ref{app:ADR_details}.

\subsubsection{ADR Results}
The reaction graph performance for the \ac{ADR} model experiments was highly effective, with a median \( F_1 \) score of \( 1.0 \), indicating perfect graph estimation in most cases. A small subset of experiments had \( F_1 \) scores below \( 0.8 \), and two experiments resulted in \( F_1 = 0.0 \) due to no links being estimated. These results highlight \mcastle's effectiveness at recovering reaction dynamics in recoverable regimes, while cases with lower \( F_1 \) scores underscore the inherent challenges of causal inference in regimes where reaction signals are weak or absent. For a detailed frequency distribution of \( F_1 \) scores, see Appendix \ref{app:adr_f1_scores}. The appendix includes a histogram summarizing the distribution, which shows a mean \( F_1 \) score of \( 0.912 \) and highlights the rarity of cases with \( F_1 = 0 \).

\Cref{fig:angle_error} presents angle estimation performance using \mcastle's estimated causal stencil graphs as a function of diffusion coefficient (left) and advection velocity (right). The median error across all experiments was \( 4.76^\circ \), demonstrating \mcastle's generally strong performance. For diffusion coefficients, mean errors remain relatively stable around \( 20^\circ \) before increasing at higher values, while median errors stay consistently lower around \( 5-15^\circ \). The proportion of large errors (\( > 45^\circ \)) increases steadily from near \( 0\% \) to approximately \( 20\% \) as diffusion increases, indicating that higher diffusion coefficients are associated with greater ambiguity in advection dynamics.

\mcastle demonstrates robust performance across a range of diffusion and advection regimes. While the most difficult settings increase the rate of large errors, results show that error is generally low in all cases. The observed increase in errors in certain regimes (e.g., high diffusion or extreme advection velocities) is consistent with the expected behavior of the system, where dynamics become less recoverable due to the fundamental nature of the underlying processes. The interplay between diffusion, advection, and reaction significantly impacts causal inference accuracy. Strong advection improves directional clarity, while high diffusion introduces ambiguity that makes advection dynamics less distinguishable or absent. Reaction dynamics are recoverable in most cases, with rare exceptions where signals are weak or absent.

\section{Real-World Applications}
\label{sec:case_studies}
We next evaluate \mcastle in two real-world settings: atmospheric chemistry in a fully coupled Earth system model and observed \ac{ENSO} dynamics in observation data. The atmospheric chemistry study uses data from \ac{E3SMv2-SPA} to demonstrate reconstructing chemical pathways in a real-world application setting. Our study of \ac{ENSO} shows that \mcastle recovers physically consistent, phase-dependent ocean–atmosphere coupling, demonstrating its utility for interpreting causal structure in large-scale Earth system phenomena.

\subsection{Atmospheric Chemistry}
\label{sec:atmospheric_chemistry}
To explore realistic physical dynamics with well-understood underlying mechanisms, we examine atmospheric aerosol chemistry following the 1991 Mount Pinatubo eruption. This volcanic event injected approximately 20 Tg of \ce{SO2} into the atmosphere, forming stratospheric sulfate aerosols that persisted, advecting and diffusing globally, for approximately two years~\citep{parker1996,Guo.2004.10.1029/2003gc000654}. The resulting climate perturbation provides a natural analogue for stratospheric aerosol injection climate strategies.

We analyze a Mount Pinatubo eruption simulation produced by the fully coupled \ac{E3SMv2-SPA} model, which includes prognostic sulfate chemistry for key species such as \ce{SO2}, \ce{H2SO4}, and \ce{SO4}. The intrinsic representation of aerosols within this modeling framework introduces additional complexity to the chemical pathway interpretation and analysis. For additional background on the Pinatubo eruption, its climate impacts, and a description of the model components and experimental setup, refer to Appendix \ref{sec:appendix_atmospheric_chemistry}.

\subsubsection{Experimental Setup}
We examine the chemical pathway of tagged Mt. Pinatubo-born \ce{SO2} aerosols. These aerosols follow the pathway $\ce{SO2} \rightarrow \ce{H2SO4} \rightarrow \ce{SO4}$ and $\{\ce{SO2}, \ce{H2SO4}, \ce{SO4}\} \rightarrow \text{FSDS}$, which mediates incoming solar radiative flux and impacts atmospheric temperature~\citep{Ramachandran.2000.10.1029/2000jd900355,Hu.2024.10.5194/egusphere-2024-2227}. For detailed information on the dataset, spatial resolution, and temporal sampling, see Appendix \ref{sec:appendix_atmospheric_chemistry}.

Despite coarse temporal sampling, we demonstrate that much of the chemical pathway is recoverable using \mcastle's spatial replicates. The dataset includes $N \times N = 30 \times 30$ grid cells with $T=7$ time samples per species per grid cell, resulting in $(N-2)^2T=5,488$ samples per species per \ac{LENS} node. Because atmospheric aerosols mediate radiative flux at rates far exceeding the $6$\,h temporal resolution, we apply an \mcastle{} \emph{link assumption} that blacklists outgoing edges from FSDS, encoding the prior knowledge that FSDS cannot appear as a lag-1 parent of any chemical species. We evaluate \mcastle{} both with and without this assumption to isolate its effect on recovered structure.

\subsubsection{Results: Mt. Pinatubo Pathway Recovery, Compared to \ac{PCMCI}}
\label{subsubsec:chem_results}

To evaluate recovered structure against prior knowledge, we constructed a reference graph encoding the causal pathways expected from atmospheric chemistry at $6$\,h resolution (Figure~\ref{fig:four_panel_causal_comparison}(a)). Solid arrows denote the primary chain $\ce{SO2} \rightarrow \ce{H2SO4} \rightarrow \ce{SO4} \rightarrow \text{FSDS}$ and lag-1 autocorrelations, each expected to manifest within a single timestep. Dashed arrows denote two mediated pathways, $\ce{SO2} \rightarrow \ce{SO4}$ and $\ce{H2SO4} \rightarrow \text{FSDS}$, which are physically plausible but whose appearance within one step depends on how completely the intermediate step has proceeded. $F_1$ is computed treating both solid and dashed edges as true positives.

\mcastle{} recovers the reference structure nearly perfectly, with
$F_1 = 0.95$ under the FSDS link assumption (Figure~\ref{fig:four_panel_causal_comparison}(b)) and $F_1 = 0.86$ without (Figure~\ref{fig:four_panel_causal_comparison}(c)). The only discrepancy in the constrained run is a weak $\ce{SO2} \rightarrow \text{FSDS}$ edge. This edge may reflect a fraction of \ce{SO2} having completed the full chain within a single timestep, or it may be a false positive. We treat it conservatively as the latter, so that the reference remains an a priori mechanistic expectation rather than a tuned target. Without the FSDS assumption, \mcastle{} additionally recovers $\text{FSDS} \rightarrow \ce{SO4}$ and $\ce{SO4} \rightarrow \ce{H2SO4}$. The first is the physically inadmissible edge the assumption is designed to exclude. The second reverses a step of the oxidation chain and does not involve FSDS directly, illustrating that unconstraining one node's outgoing edges propagates through the conditioning sets used during causal discovery and degrades structure recovery beyond the blacklisted node itself.

For comparison, we applied \ac{PCMCI} directly to the area-weighted
spatial mean of each variable (Figure~\ref{fig:four_panel_causal_comparison}(d)). The result is an empty graph. Spatial aggregation collapses $5{,}488$ samples per species into $7$, leaving too few effective degrees of freedom for \ac{PCMCI} to reject independence. \mcastle{} avoids this collapse by treating each grid cell's neighborhood as a replicate of the same spatial stencil. For short-duration, spatially-structured phenomena like volcanic plume evolution, exploiting spatial structure is a prerequisite for causal recovery.

\begin{figure}[ht]
    \centering
    \includegraphics[width=\textwidth]{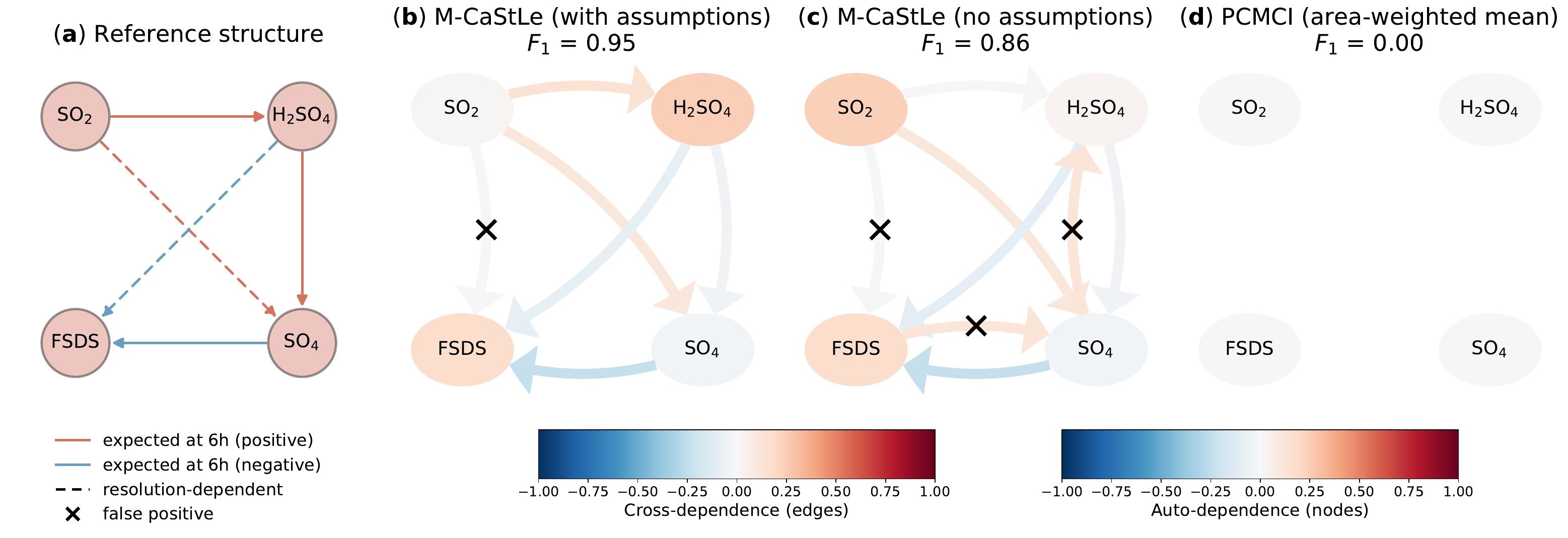}
    \caption{Causal graphs recovered from six-hourly \ac{E3SMv2-SPA} output over the Pinatubo plume region ($0^\circ$--$30^\circ$\,N, $60^\circ$--$90^\circ$\,E) for four Pinatubo-attributable variables: \ce{SO2} burden, \ce{H2SO4} burden, \ce{SO4} \ac{AOD}, and downwelling surface shortwave flux (FSDS). \textbf{(a)} Reference structure encoding expected causal pathways from atmospheric chemistry at $6$\,h resolution: solid arrows denote cross-dependencies expected within a single timestep; dashed arrows denote resolution-dependent relationships that may or may not resolve within one step; node shading indicates expected positive auto-dependence (persistence). All count as true positives for $F_1$. \textbf{(b)} \mcastle with link assumptions blacklisting outgoing edges from FSDS. \textbf{(c)} \mcastle without link assumptions. \textbf{(d)} \ac{PCMCI} applied to the area-weighted spatial mean of each variable. In panels (b)--(d), edge color encodes cross-dependence strength and node color encodes auto-dependence strength, with positive values in reds and negative in blues; black crosses mark false positives. \mcastle correctly identifies the general pathway $\ce{SO2} \rightarrow \ce{H2SO4} \rightarrow \ce{SO4}$ and $\{\ce{SO2}, \ce{H2SO4}, \ce{SO4}\} \rightarrow \text{FSDS}$, with a single weak false positive ($\ce{SO2} \rightarrow \text{FSDS}$) in the constrained run; the unconstrained run recovers additional implausible edges, while \ac{PCMCI} on the aggregated timeseries recovers no structure at all.}
    \label{fig:four_panel_causal_comparison}
\end{figure}

\subsection{\acl{ENSO} Dynamics}
\label{sec:enso}
The \acf{ENSO} is a quasi-periodic Earth system phenomenon arising from coupled ocean–atmosphere interactions in the tropical Pacific, with far-reaching impacts on global weather and climate. \ac{ENSO}’s positive phase (El Ni\~{n}o) is characterized by warm \ac{SST} anomalies over the central and eastern tropical Pacific, which enhance convection and rainfall over the central Pacific and favor anomalously dry conditions over the western Pacific. In contrast, \ac{ENSO}’s negative phase (La Ni\~{n}a) features cold \ac{SST} anomalies over the central and eastern tropical Pacific, suppressing convection there while the Indo-Pacific warm pool becomes anomalously warm~\citep{philander_nino_1983}. Although the full set of processes governing \ac{ENSO} evolution remains an active area of research, the causal influence of \ac{SST} anomalies on tropical Pacific convection is well established and therefore provides a useful test case for evaluating \mcastle.

We applied \mcastle-\ac{DYNOTEARS} to two contrasting periods corresponding to an El Ni\~{n}o event and a La Ni\~{n}a event, using daily anomalies of \ac{SST} and \ac{OLR} (used here as a continuous proxy for rainfall) from ERA5 reanalysis data~\citep{Hersbach.2020.10.1002/qj.3803} at $1^\circ$ spatial resolution. See Appendix \ref{app:enso_alg_details} for \mcastle implementation details. The resulting causal maps are shown in \Cref{fig:enso_dynotears}. In both events, \mcastle identifies \ac{SST} as exerting causal influence on both \ac{SST} (persistence) and \ac{OLR}, while \ac{OLR} primarily drives itself, with no robust evidence for \ac{OLR} driving \ac{SST}. This asymmetry is physically consistent with the known ocean–atmosphere coupling in the tropical Pacific, where \ac{SST} anomalies modulate convection rather than the reverse~\citep{trenberth_relationships_2005}.

\begin{figure}[ht]
    \centering
    \includegraphics[width=\columnwidth]{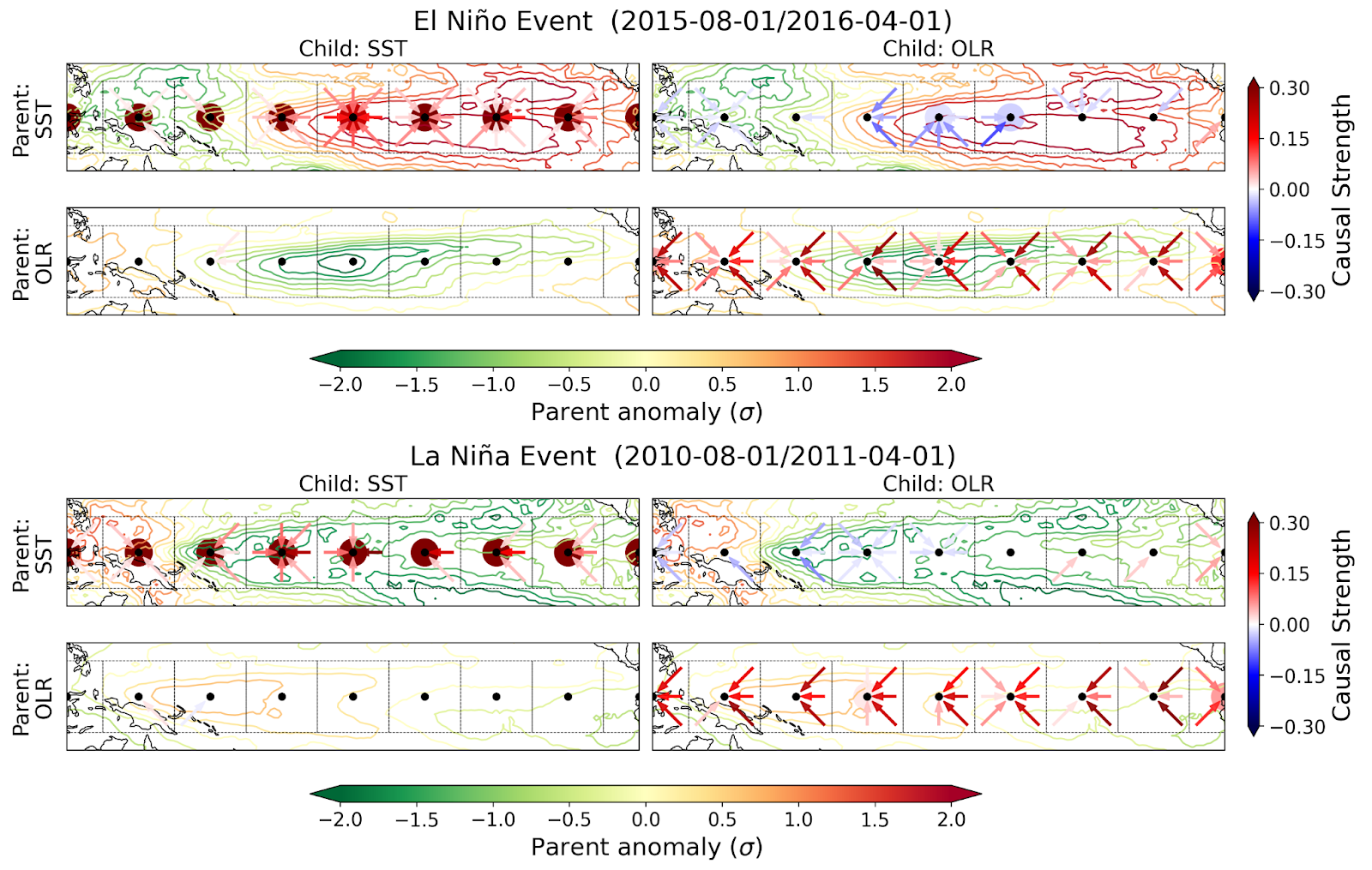}
    \caption{\mcastle results obtained using \ac{DYNOTEARS} illustrating causal relationships between \acf{SST} and \acf{OLR} over the tropical Pacific for two contrasting periods corresponding to an El Ni\~{n}o and a La Ni\~{n}a event. Background contours show mean anomalies over each period relative to the 1995–2024 climatology (green contours indicate negative anomalies; red contours indicate positive anomalies). \mcastle was applied over $20^\circ \times 20^\circ$ regions centered along the equator (dashed boxes), and the inferred local causal stencils are overlaid for each region. Arrow and center-dot colors denote the magnitude and sign of the inferred causal strength from \ac{DYNOTEARS}.}
    \label{fig:enso_dynotears}
    % \Description[<short description>]{<long description>}
\end{figure}

Moreover, the magnitude and spatial structure of the inferred causal links reflect known \ac{ENSO}-phase-dependent behavior. During El Ni\~{n}o, positive \ac{SST} anomalies over the central Pacific are associated with negative \ac{OLR} anomalies (enhanced convection), whereas during La Ni\~{n}a the \ac{SST}–\ac{OLR} coupling is markedly weaker and consistent with suppressed convection over the central and eastern Pacific. In addition to these cross-variable links, \mcastle captures differences in univariate dynamics between ENSO phases. During the La Ni\~{n}a event, both \ac{OLR} and \ac{SST} anomalies exhibit predominantly westward propagation across the tropical band, while during the El Ni\~{n}o event, east–west propagation of \ac{SST} anomalies is largely confined to the eastern tropical Pacific and the zonal propagation of \ac{OLR} anomalies is comparatively weaker.

The correspondence between the inferred causal structures and well-established \ac{ENSO} dynamics, together with the clear phase-dependent differences, demonstrates that \mcastle can be meaningfully applied to observed data for specific events. These results highlight the potential of \mcastle to help identify and interpret causal relationships in the Earth system that give rise to particular—and potentially extreme—climate conditions.

\section{Discussion}
\label{sec:discussion}
We introduced \mcastle, a multivariate generalization of the space-time causal discovery meta-algorithm \castle~\citep{Nichol.2025.10.1029/2024jh000546}. \mcastle generalizes the \acf{LENS} and \acf{PIP} components of \castle to jointly recover within-variable, cross-variable, and spatial causal relationships in gridded space-time data, producing a multivariate causal stencil graph that captures both transport and inter-variable dynamics. When each grid cell contains a single variable, this formulation reduces to the original univariate \castle framework. To aid interpretation, we proposed a decomposition method that separates spatial and inter-variable relationships into a \emph{spatial graph} and \emph{reaction graph}.

\subsection{Key Findings}
Our multivariate \ac{VAR} benchmarks in \Cref{sec:VAR} demonstrated \mcastle's high precision and recall, outperforming the \ac{PC}, \ac{PCMCI}, and \ac{DYNOTEARS} causal discovery algorithms in multivariate gridded systems. \mcastle retains high precision ($>0.9$ across all $V$) and F$_1$ up to $\approx 0.96$ for $V = 1$, degrading moderately as the number of variables increases (\Cref{fig:precision_recall_f1_castle_vs_cart_with_baselines}). Direct non-spatial baselines remain near chance throughout most of the tested regime, while Cartesian-\castle improves over them but remains consistently below \mcastle. The recall decline with $V$ is consistent with, and expected under, the stability constraint forcing coefficients toward $0$ (the piranha constraint; \citealp{Tosh.2025.10.1090/noti3044}), rather than false positives.

To evaluate general applicability, we tested \mcastle on realistic physical systems, including the \acf{ADR} \ac{PDE} model, in \Cref{sec:adr_dynamics}. \mcastle accurately recovered spatial and inter-species dynamics across diverse advection-diffusion environments, with median spatial angle error below \(5^\circ\) and perfect reaction-graph recovery in more than 90\% of cases (\Cref{fig:angle_error}). Error increases at high diffusion where transport direction is physically less identifiable.

In the atmospheric chemistry case, in \Cref{sec:atmospheric_chemistry}, \mcastle successfully identified the causal pathway from \ce{SO2} to solar radiative flux following the Mt. Pinatubo eruption, aligning with known atmospheric impacts. Despite only seven temporal snapshots, \mcastle reconstructs the expected \(\ce{SO2}\!\!\rightarrow\!\!\ce{H2SO4}\!\!\rightarrow\!\!\ce{SO4}\) sequence and associated radiative impacts (\Cref{fig:four_panel_causal_comparison}), demonstrating the power of spatial replication. Pooling $(N-2)^2$ neighborhoods yields $5,488$ window samples (up to a design-effect reduction under spatial dependence).

Finally, we demonstrated \mcastle's ability to recover known event-scale, regional coupling patterns within \ac{ENSO} phases in \Cref{sec:enso}. Phase-dependent \ac{SST} $\rightarrow$ \ac{OLR} causal links match canonical El Ni\~{n}o vs. La Ni\~{n}a coupling patterns (\Cref{fig:enso_dynotears}), while \ac{OLR} $\rightarrow$ \ac{SST} paths remain negligible. These results suggest \mcastle can recover physically consistent, event-scale local coupling patterns in observed Earth system data.

\subsection{Limitations \& Practical Recommendations}
\mcastle inherits \castle's reliance on locality and (approximate) stationarity (Assumptions T2 \& S2), so it is not appropriate for every space-time problem.

\begin{itemize}[leftmargin=*, itemsep=2pt]
    \item \textbf{Nonlocal structure / teleconnections:} \mcastle targets local mechanisms and is not intended to infer global teleconnection networks. While \ac{ENSO} is a teleconnection phenomenon, our analysis focuses on local, phase-dependent \ac{SST}--\ac{OLR} coupling within $20^\circ \times 20^\circ$ windows (\Cref{sec:enso}). For teleconnections, dimensionality-reduction approaches are often more suitable~\citep{Runge.2019.10.1038/s41467-019-10105-3,Tibau.2022.10.1017/eds.2022.11}.
    \item \textbf{Regime changes:} apply \mcastle only within time intervals and regions where one dominant causal structure is plausible; otherwise estimated stencils may mix regimes.
    \item \textbf{Resolution limits:} coarse temporal (and/or spatial) sampling can obscure fast processes (e.g., Mt.\ Pinatubo radiative feedbacks; \Cref{subsubsec:chem_results}). However, \citet[Appendix I]{Nichol.2025.10.1029/2024jh000546} suggests \castle remains informative under substantial spatial coarsening, so precise space-time alignment is not always required.
    \item \textbf{Correlated spatial replicates:} if nearby grid cells are highly correlated, pooling yields fewer effective samples and dynamics may be harder to identify (see \Cref{sec:math_foundations}).
    \item \textbf{Mitigations and extensions:} larger/adaptive neighborhoods may help at the cost of fewer replicates; the current implementation assumes rectangular grids but it may generalize to irregular meshes with adapted locality constraints.
    \item \textbf{Preprocessing and constraints:} climatological anomalies~\citep{Hansen.1999.10.1029/1999JD900835} (as in \Cref{sec:enso}), detrending, or lag-differences can reduce seasonal confounding and improve stationarity~\citep{Runge.2023.10.1038/s43017-023-00431-y}. Link assumptions (black-/white-lists) can encode known structures (e.g., FSDS as a sink in \Cref{sec:atmospheric_chemistry}).
    \item \textbf{Regularization and hyperparameters:} \mcastle inherits regularization choices from the chosen \ac{PIP} algorithm (e.g., \ac{PC}/\ac{PCMCI} significance levels and multiple-testing thresholds, or \ac{DYNOTEARS} sparsity penalties). These parameters control edge inclusion and are not tunable by predictive cross-validation, since the objective is causal structure rather than one-step-ahead forecast error. We therefore recommend checking robustness via sensitivity/stability analyses. Appendices \ref{app:data_generation}, \ref{app:ADR_details}, \ref{app:chem_alg_details}, and \ref{app:enso_alg_details} provide our per-experiment \ac{PIP} configurations.
\end{itemize}

\subsection{Future Work \& Conclusion}
As noted above, generalizing \mcastle's $3\times 3$ stencil to larger spatial neighborhoods, such as radius-2 Moore neighborhoods could improve causal discovery in systems with mismatched spatial and temporal resolutions~\citep{Nichol.2023.10.2172/1991387}. For irregular or unfixed meshes, adaptive neighborhood structures could be developed too. For more dynamic settings, in which spatial stationarity is less certain or quickly changing, stationary regions could be identified via stability checks with \mcastle: iteratively updating region sizes and locations to find consistent stencils.

This work showcases \mcastle's ability to discover multivariate grid-level space-time causal structures in four complementary settings: synthetic \ac{VAR} benchmarks, a physics-based \ac{ADR} verification problem, an atmospheric chemistry case study, and observed \ac{ENSO} dynamics. Together, our findings show that \mcastle’s local-stencil view preserves scientifically interpretable structure across a spectrum of dynamical regimes. By generalizing \castle from univariate to multivariate systems, \mcastle broadens local stencil learning from single-field analyses to joint causal discovery over multiple co-located variables while retaining the original method as a special case.

% \section*{Broader Impact Statement}

% In this optional section, TMLR encourages authors to discuss possible repercussions of their work,
% notably any potential negative impact that a user of this research should be aware of. 
% Authors should consult the TMLR Ethics Guidelines available on the TMLR website
% for guidance on how to approach this subject.

% \section*{Author Contributions}
% If you'd like to, you may include a section for author contributions as is done
% in many journals. This is optional and at the discretion of the authors. Only add
% this information once your submission is accepted and deanonymized. 

\section{Acknowledgments}
We thank Laura P. Swiler for her significant guidance and feedback throughout the development of this work. We thank Diana Bull, the Principal Investigator of the CLDERA (CLimate impact: Determining Etiology thRough pAthways) project at Sandia National Laboratories (SNL), along with the entire CLDERA team, for their support in making this work possible. We also thank Joey Hart and Shane McQuarrie at SNL for their contributions to developing the initial ADR model.

This work was supported by the Laboratory Directed Research and Development program at Sandia National Laboratories, a multi-mission laboratory managed and operated by National Technology \& Engineering Solutions of Sandia, LLC (NTESS), a wholly owned subsidiary of Honeywell International Inc., for the U.S. Department of Energy’s National Nuclear Security Administration (DOE/NNSA) under contract DE-NA0003525. This written work is authored by employees of NTESS. The employees, not NTESS, own the right, title, and interest in and to the written work and is responsible for its contents. Any subjective views or opinions that might be expressed in the written work do not necessarily represent the views of the U.S. Government. The publisher acknowledges that the U.S. Government retains a non-exclusive, paid-up, irrevocable, world-wide license to publish or reproduce the published form of this written work or allow others to do so, for U.S. Government purposes. The DOE will provide public access to results of federally sponsored research in accordance with the DOE Public Access Plan.

This paper describes objective technical results and analysis. Any subjective views or opinions that might be expressed in the paper do not necessarily represent the views of the U.S. Department of Energy or the United States Government.

\bibliography{fullbib}
\bibliographystyle{tmlr}

\appendix

\section{Univariate \castle Preliminaries}
\label{app:u_castle}
Here, we outline the foundational framework on which our method builds. In depth detail, analysis, and discussion of applications and limitations are given by \citet{Nichol.2025.10.1029/2024jh000546}.

\castle leverages local causal regularities to transform the causal discovery problem. It converts a high-dimensional space with many graph nodes and limited observations into an embedding with fewer nodes and more abundant observations. The present work generalizes \castle, broadening its applicability to multivariate space-time dynamics, expanding its utility across diverse space-time systems in the physical sciences.

% \begin{figure}[ht]
%     \centering
%     % \includegraphics[width=0.5\columnwidth]{LENS.pdf}
%     \includesvg[width=\columnwidth]{castle_LENS_3panel.svg}
%     \caption{Illustration of the \acf{LENS} transformation applied to a $4\times4$ grid. \textbf{Panel A}: The $3\times3$ Moore neighborhood is extracted around each of the four interior cells (colored circles), sliding across the grid in raster order (dashed arrows). The colored overlays show the four extracted neighborhoods, each centered on its respective interior cell. \textbf{Panel B}: For each of the nine relative neighborhood positions (NW through SE, indicated by the small icons), the time series from all four neighborhoods are concatenated along the time axis, multiplying the effective sample size by the number of interior cells. \textbf{Panel C}: The nine concatenated time series are arranged into their corresponding $3\times3$ spatial positions, forming the \ac{LENS}. The four colors in each cell denote the contributions from each of the four source neighborhoods. \Cref{fig:ucastle} illustrates how the \ac{LENS} is used in the full \castle process.}
%     \label{fig:LENS}
% \end{figure}

In many natural and engineered systems, complex global behaviors emerge from simple local interactions that follow consistent physical dynamics. \citet{Nichol.2025.10.1029/2024jh000546} called such systems \textit{\ac{PDE}-like} because they exhibit consistent dynamics defined by interactions between adjacent points in space, with smooth transitions between dynamical boundaries and equilibria. \castle makes two sets of assumptions that constrain its application domain: (1) spatial and temporal locality and (2) spatial and temporal stationarity. In short, (1) holds causal dynamics to be most direct and impactful between adjacent grid cells and adjacent points in time; near things impact near things first. (2) ensures that one causal structure is sought for the given spatial domain and temporal interval. These constraints encompass numerous well-studied systems including those governed by \acp{PDE}, cellular automata, and various lattice models in statistical physics. These assumptions are generalized to multiple variables and discussed more precisely in \Cref{sec:methods}.

\castle leverages locality and stationarity to collect time series representing the space-time causal replicates in such systems. Every grid cell's time series encodes the causal influence of its neighbors, and each can be used as an informative replicate of the system's local dynamics. \castle processes a set of grid cells, collecting each one's data on its local dependence, then learns the causal structure of the grid cells and their neighborhoods. In some cases, causal discovery can be applied independently to groups of local grid cells. However, many systems contain more grid cells than observations within each cell. \castle effectively estimates causal structures in this regime by efficiently using all available dynamical information in a region.

\subsection{Phase 1: Locally Encoded Neighborhood Structure (LENS)}

\castle's first phase forms the \ac{LENS}, an embedding representing the \emph{Moore neighborhood}, a 3$\times$3 matrix of a grid cell and its eight immediate neighbors. The \ac{LENS} contains concatenated time series from each grid cell's Moore neighborhood such that the local dynamics of each neighbor are repeated. For an $N\times N$ grid, each entry of the embedding contains $N_\text{rep}=(N-2)^2$ concatenated time series collected from each interior grid cell. Thus, each time series has length $TN_\text{rep}$ for $T$ time samples per grid cell. Unlike dimensionality reduction techniques, no information is marginalized.

% \begin{figure*}[ht]
%     \centering
%     \includegraphics[width=\textwidth]{uCaStLe_Ex.pdf}
%     \caption{The \castle computational procedure on an example $4$$\times$$4$ grid: the \ac{LENS} construction (Phase 1) and \ac{PIP} analysis (Phase 2) produce a causal stencil graph capturing local dynamics.}
%     \label{fig:ucastle}
%     % \Description[Complete CaStLe workflow from input data to causal stencil graph]{A comprehensive illustration of the CaStLe algorithm workflow applied to a 4×4 gridded space-time system. The diagram shows three main stages: (1) Input Data - a 4×4 grid where each cell contains time series data represented by different colored waveforms; (2) Phase 1: LENS Construction - demonstrating both "Neighborhood Collection" where Moore neighborhoods are systematically extracted from each interior grid cell, and the resulting "LENS" structure with concatenated time series organized by spatial position (NW, N, NE, W, C, E, SW, S, SE); and (3) Phase 2: PIP - showing the final "Causal Stencil Graph" with nodes representing grid positions and directed edges indicating discovered causal relationships, with a prominent connection to the center node.}
% \end{figure*}

\subsection{Phase 2: Parent Identification Phase (PIP)}

Once the \ac{LENS} embedding is constructed, \castle's second phase, the \ac{PIP}, applies an adapted causal discovery algorithm to the embedding. Any time series causal discovery algorithm can be adapted for \castle by treating the embedding's center grid cell as special. Specifically, the variable at the center position serves as the only child in the resulting causal graph, while parent sets remain unconstrained. This small adaptation has multiple significant effects: it creates a graph of the generalized ancestry for each grid cell, eliminates would-be unobserved confounding between the embedding's outer grid cells and their neighbors beyond the embedding, and increases computational and statistical efficiency.

Resulting from applying the \ac{PIP} on the \ac{LENS} is the \textit{causal stencil graph}, a representation of the local causal dynamics between all grid cells in the system. \castle provides a powerful foundation for both forward simulation and inverse problems: identifying the underlying causal structure from observed space-time data.

\section{Detailed Methodology for Multivariate Space-Time \acp{VAR}}
\label{app:data_generation}
To adapt the space-time \ac{VAR} procedure for multivariate systems, we grow the \ac{NDM} in a new variable dimension, which gets mapped to a larger, flat, $\bm{A}$ matrix. The multivariate \ac{NDM} describes interactions between multiple variables at the local grid-level, enabling \ac{VAR} modeling of multivariate space-time dynamics. In this appendix, we use $N\times M$ to describe the most general rectangular-grid case; this specializes to the $N\times N$ notation used in the main text.

For a multivariate system on an $N\times M$ grid with $V$ variables per cell and $T$ time samples, let $\bm{X} \in \mathbb{R}^{N \times M \times V \times T}$ with elements $X_{i,j,v,t}$. The corresponding global \ac{VAR}(1) model can be written after vectorization as
\begin{equation}\label{eq:VARs}
    \bm{x}_t = \bm{A} \bm{x}_{t-1} + \bm{\eta}_t,
\end{equation}
where $\bm{A}$ is the coefficient matrix encoding linear dependencies between all variables in the system and $\bm{\eta}$ represents independent \emph{innovations} on $\bm{X}$ for each variable at each time step. In this case, innovations are modeled with a unit normal distribution.

For a system of $V$ variables, the multivariate dynamics are represented by a set of $V$$\times$$V$ $3$$\times$$3$ matrices. Each $3$$\times$$3$ matrix corresponds to the space-time dependence structure of a particular pair of parent and child variables. Like the univariate \ac{NDM}, each entry in each $3$$\times$$3$ matrix is a coefficient value representing the influence of the entry's spatial location in the Moore neighborhood on the center location.

The \ac{NDM} is mapped to an $\bm{A}$ matrix, which represents the interactions of every grid cell-variable on every other grid cell-variable. For a grid of size $N$$\times$$M$ spatial dimensions and $V$ variables, the matrix $\bm{A} \in \mathbb{R}^{NMV \times NMV}$. With the computed $\bm{A}$ matrix, we again enforce stability by ensuring $\rho(\bm{A}) < 1.0$, where $\rho(\bm{A})$ is the spectral radius of $\bm{A}$.

With a stable $\bm{A}$ matrix, experimental data can be generated for any number of grid cells, time samples, local dependencies, and variables. Although $\bm{A}$ is larger, the \acp{VAR} still have the form of Equation \ref{eq:VARs}. Since most $\bm{A}$ matrices will be unstable, our implementation uses an accept-reject scheme similar to the univariate approach of \citet{Nichol.2025.10.1029/2024jh000546} to generate stable $\bm{A}$ matrices:
\begin{enumerate}
    \item Generate a random set of \(3 \times 3\) \emph{local dynamics matrices}, \(\{C_{ij}\}\), for each pair of child and parent variables, resulting in \(V \times V\) matrices. Each \(C_{ij}\) has $k$ non-zero elements, including the central element (autocorrelation), where \(1 \leq k \leq 9\). Each of the \(k\) non-zero elements, \(\{a_i\}_{i=1}^k\), is assigned a random value satisfying \(1.0 \ge a_i \ge s_*\).
    \item Expand \(\{C_{ij}\}\) to form the matrix \(\bm{A}\) for a grid of size \(N \times M\), resulting in \(\bm{A} \in \mathbb{R}^{NMV \times NMV}\).
    \item If $|\lambda_{\max}(\bm{A})| \geq 1$, scale $\bm{A}$ by $|\lambda_{\max}(\bm{A})|$.
    \item If every nonzero entry of $\bm{A}$ has magnitude less than $s_*$, reject; otherwise accept.
\end{enumerate}
where $|\lambda_{\max}(\bm{A})|$ is the maximum absolute eigenvalue of $\bm{A}$. This is used to sample from the set of statistically stationary \& spatially homogeneous VARs on a 2D grid with minimum signal strengths $s_* \geq 0.1$ and fixed local nonzero counts in the range $k \in \{1,2,\dots,9\}$

\subsection{Metrics Definitions}
\label{app:metrics}
Since \acp{VAR} map directly to ground truth causal graphs, we measured \mcastle's performance using binary classification measures. Let $G = (V, E)$ be the ground truth graph where $V$ is the set of nodes and $E \subseteq V \times V$ is the set of edges. For any node pair $(i,j) \in V \times V$, a positive instance is defined as $(i,j) \in E$ and a negative instance as $(i,j) \notin E$. This enables our usage of precision, recall, and $\text{F}_1$ score, defined as follows:
\begin{equation}
\text{Precision} = \frac{\text{TP}}{\text{TP} + \text{FP}}
\end{equation}
\begin{equation}
\text{Recall} = \frac{\text{TP}}{\text{TP} + \text{FN}}
\end{equation}
\begin{equation}
\text{F}_1 = \frac{2 \cdot \text{Precision} \cdot \text{Recall}}{\text{Precision} + \text{Recall}}
\end{equation}
where TP, FP, TN, and FN denote true positives, false positives, true negatives, and false negatives, respectively. Put simply, precision is the proportion of correctly detected positives to all detected positives, with a range of $[0, 1]$, where $1$ is a perfect precision; recall is the proportion of correctly detected positives to how many positives should have been detected, with a range of $[0,1]$, where $1$ is a perfect recall; and the $\text{F}_1$ score is the harmonic mean of precision and recall, with a range of $[0,1]$, where $1$ indicates perfect graph estimation. To avoid divide-by-zero edge cases, we used an adapted version fully detailed by \citet{Nichol.2023.10.2172/1991387}.

\section{Completed Data Generation Parameters}
\label{app:data_generation_parameters}
As noted in \Cref{sec:VAR}, not all parameter combinations generated stable systems. Here, we present the parameter ranges that did successfully generate 30 replicates to produce our results. We additionally evaluate the range of coefficient sizes generated, demonstrating the difficulty of creating complex systems with strong signals and many interdependencies.

\begin{figure*}[ht]
    \centering
    \includegraphics[width=0.82\textwidth]{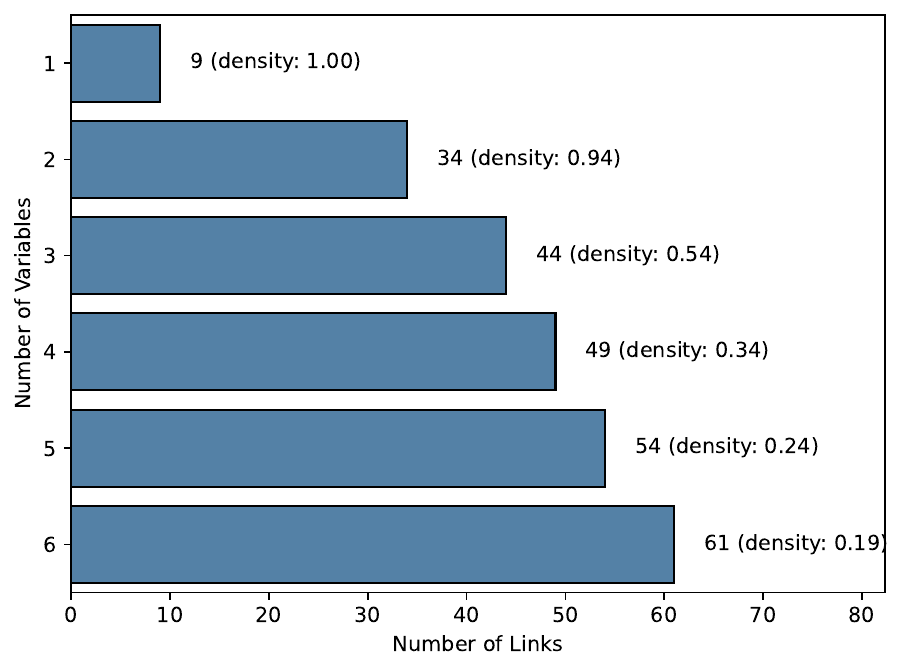}
    \caption{Extent of number of links with 30 Replicates. Horizontal bars show, for each Number of Variables (y‐axis), the maximum Number of Links (x‐axis) for which 30 independent simulations were performed; each bar is annotated to indicate the corresponding link density. Our experiments covered $4$$\times$$4$ grids and 1-6 variables per grid. All experiments used 1000 time samples and coefficient values between 0.1 and 1.0. The network density, $d \in (0, 1]$, is defined as the ratio of actual links, $E$, to the maximum possible number of links, $d = E/(9V^2)$. Not all density values produced 30 stable systems within our computational constraints, particularly at higher densities. This figure shows which parameter combinations successfully generated sufficient replicates for statistical analysis.}
    % \caption{Parameter ranges used in our experimental design, showing the link count distribution for each grid size and variable count combination. Each horizontal line represents the span of network links tested, with each parameter combination having at least 30 replicate experiments (n values shown). Our experiments covered grid sizes from $4$$\times$$4$ to $10$$\times$$10$ and 1-6 variables per grid. All experiments used 1000 time samples and coefficient values between 0.1 and 1.0. The network density, $d$, defined as the ratio of actual links, $E$, to maximum possible links $d = \frac{L}{(3 \times 3 \times V^2)}$, where $d \in (0, \dots 0.5]$. Not all density values produced 30 stable systems within our computational constraints, particularly at higher densities. This visualization shows which parameter combinations successfully generated sufficient replicates for statistical analysis.}
    \label{fig:param_ranges}
    % \Description[<short description>]{<long description>}
\end{figure*}

% Parameter ranges used in our experimental design, showing the link count distribution for each grid size and variable count combination. Each horizontal line represents the span of network links tested, with each parameter combination having at least 30 replicate experiments (n values shown). Our experiments covered grid sizes from 4×4 to 10×10 and 1-6 variables per grid. All experiments used 1000 time samples and coefficient values between 0.1 and 1.0. The network density, defined as the ratio of actual links ($E$) to the maximum possible links in a $3\times3$ stencil graph ($d = E/(9V^2)$), ranged from near zero to $0.5$. Not all theoretical density values produced 30 stable systems within our computational constraints, particularly at higher densities. This visualization shows which parameter combinations successfully generated sufficient replicates for statistical analysis.

\section{Additional \ac{VAR} Results}
\label{app:additional_var_results}
In this appendix, we present additional results related to the performance of our proposed method, \mcastle, with \ac{VAR} benchmarks. We delve into various metrics that evaluate the effectiveness of \mcastle.

\subsection{Exploring Recall}
\label{app:exploring_recall}
As discussed in \Cref{sec:VAR}, the decline in F$_1$ with increasing $V$ is driven primarily by reduced recall rather than reduced precision. Here, we provide supporting results showing that this behavior is consistent with weakening signal strength in stable, increasingly dense systems.

\Cref{fig:max_min_link_coefs} shows how the magnitudes of realized nonzero coefficients change as the number of links increases in the stable \ac{VAR} generator. Both the maximum and minimum coefficients decrease systematically with graph density, indicating that stability increasingly forces effects toward zero in more complex systems.

\begin{figure*}[ht]
    \centering
    \includegraphics[width=0.82\textwidth]{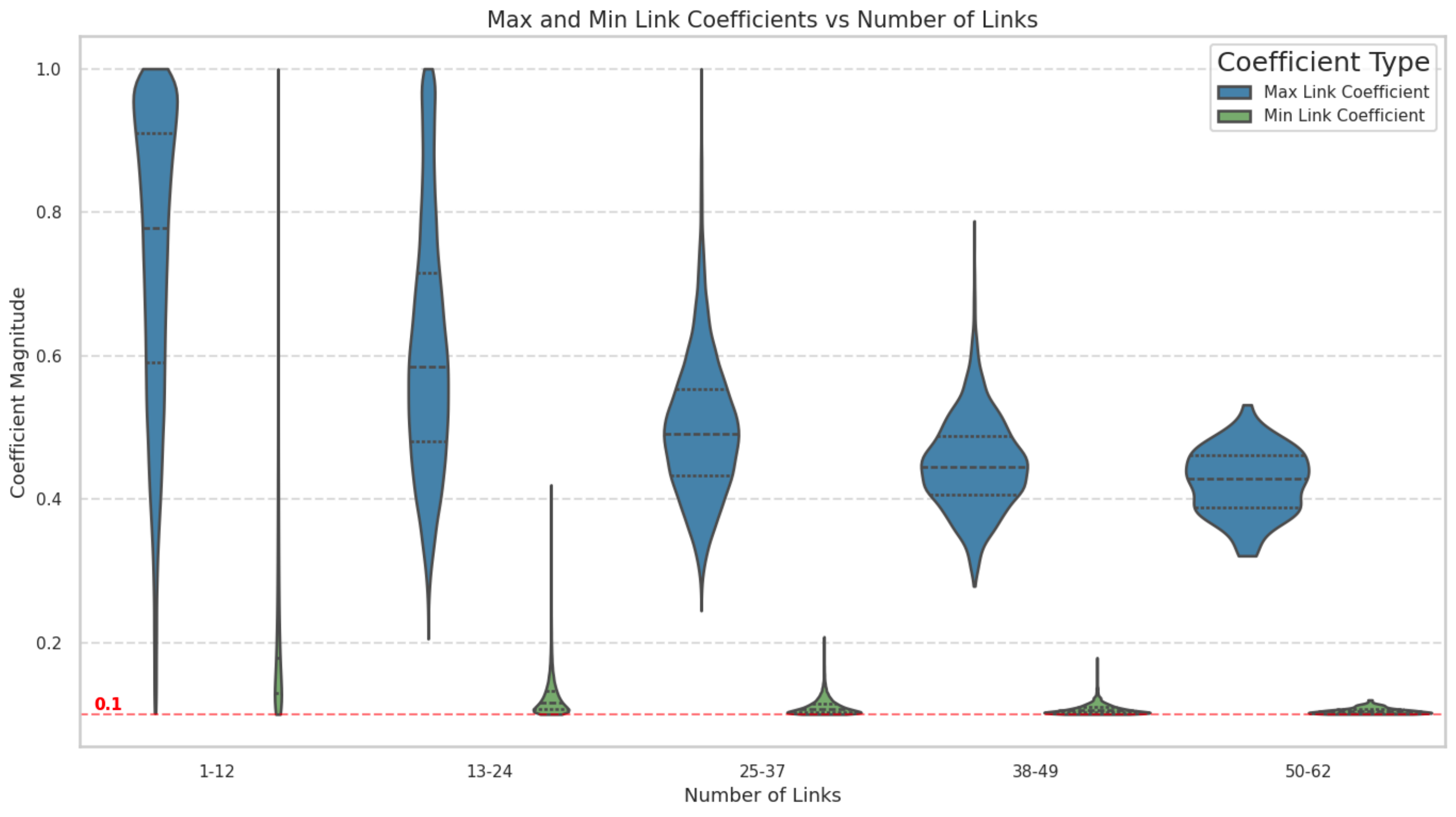}
    \caption{The relationship between link coefficients and the number of links present. As the number of links increases, maximum (blue) and minimum (green) link coefficients show a clear decreasing trend, with their distribution becoming narrower and centered around lower values. This reveals that networks with more links have weaker signals, suggesting that highly interconnected systems cannot be stable with large dependencies.}
    \label{fig:max_min_link_coefs}
\end{figure*}

To better evaluate the reason for \mcastle's relatively low recall, we tested on a separate set of benchmark systems. In these, we constructed simple systems with many more variables and a range of coefficient magnitudes. The systems model a chain of dependence across variables, with one incoming link per variable. Specifically, variable $v$ depends on exactly one spatially local parent in variable $v-1$ (i.e., $1\to2\to\cdots\to V$), replicated at every grid cell with a fixed coefficient per realization. The link is assigned a random parent location in the Moore neighborhood, and points to the center of the next variable. With this, we model different spatial relationships between variables but only one between variables. We explored $V \in \{10, 50, 100, 200\}$ and a set of 20 coefficients $\{c_i\}_{i=0}^{19}$ logarithmically spaced from 0.01 to 2.0, where $c_i = 0.01 \times 10^{2i/19}$. Every link had the same coefficient for each realization. Each realization had $T=1000$ time samples and we restricted the grid size to $4$$\times$$4$, which is the most challenging for \mcastle because there are fewer spatial replicates to leverage.

\begin{figure*}[ht]
    \centering
    \includegraphics[width=0.82\textwidth]{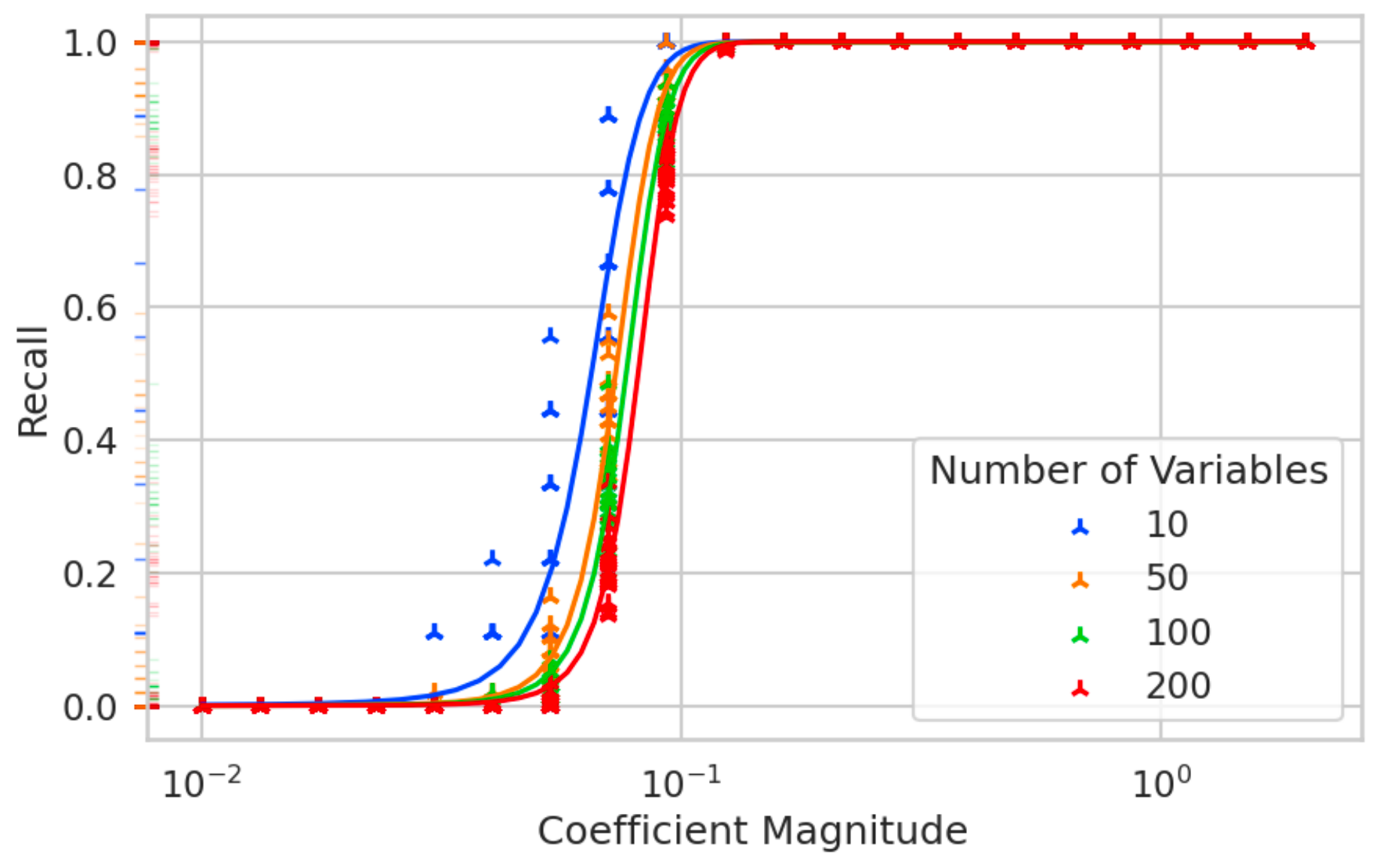}
    \caption{In simple chains of multivariate stencils, even with an extremely large number of variables, recall can be captured perfectly if the signal strength is large enough.}
    \label{fig:recall_coefs}
    % \Description[<short description>]{<long description>}
\end{figure*}

\Cref{fig:recall_coefs} illustrates that recall increases proportionately with coefficient magnitude for all numbers of variables. Recall is $0$ when coefficients are too small and $1$ when they are large enough. There is an inflection interval in the coefficient magnitudes in which recall increases sharply. The three-parameter sigmoid functions fit to each set of $V$s shows that recall is ordered by $V$. That means that, while high recall is achievable for up to 200 variables, systems with more variables are marginally more challenging to estimate, which conforms to our expectations. These results show that high recall is possible in high variable regimes if signals are strong enough.

The declining performance of \mcastle on large stable systems likely reflects a fundamental limit imposed by the stability constraint itself. For a dense transition matrix $A$ to satisfy $\rho(A) < 1$, the circular law requires that interaction strengths shrink with the square root of model complexity~\citep{Geman.1986.10.1214/aop/1176992372, MAY.1972.10.1038/238413a0}: as the number of interactions grows, each individual coefficient must approach zero to prevent the spectral radius from exceeding the stability boundary. This is not a property of any 
particular estimation method but a geometric consequence of how eigenvalues 
accumulate in large random matrices.

The statistical implication is that interaction strengths eventually fall below the noise floor of any finite-sample estimator. Because the signal available to a learning algorithm scales with the magnitude of the coefficients it must recover, and because that magnitude is bounded above by a quantity that vanishes with model complexity, the signal-to-noise ratio degrades monotonically as the system grows. This is consistent with hardness results in the system identification literature showing that the minimum eigenvalue of the controllability Gramian, which governs how much information a trajectory carries about the transition matrix, can itself vanish with model complexity~\citep{Tsiamis.2021.10.48550/arxiv.2104.01120}, and with the piranha constraint showing that many simultaneously large causal effects are geometrically incompatible with a bounded-variance process~\citep{Tosh.2025.10.1090/noti3044}.

Whether this constitutes a formal detection impossibility result for causal discovery in dense stable systems, and whether there exists a complexity threshold beyond which no algorithm can distinguish true causal structure from noise at any finite sample size, remains, to our knowledge, an open question.

\section{Detailed Information on the ADR Model and Experiments}
\label{app:ADR_details}
\subsection{Parameter Descriptions}
The \acf{ADR} model dynamics are influenced by several key parameters:
\begin{itemize}[leftmargin=*, itemsep=2pt]
    \item \textbf{Advection Velocity} ($v$): The speed at which the species are transported through the medium.
    \item \textbf{Advection Angle} ($\theta$): Determines the direction of transport for both species.
    \item \textbf{Diffusion Rate} ($D_1, D_2$): Quantifies the tendency of the species to spread out due to concentration gradients. A higher diffusion rate results in a more rapid dispersion of the species in the spatial domain.
    \item \textbf{Reaction Rate} ($R_1, R_2$): Govern the rates at which species $u_1$ decays into species $u_2$.
\end{itemize}

The reaction terms $R_1(u_1, u_2)$ and $R_2(u_1, u_2)$ govern the interaction between the two chemical species. Specifically, species $u_1$ undergoes a linear decay, converting into species $u_2$ at a rate controlled by two parameters: the decay rate ($\alpha$) and a conversion factor ($\beta$) that determines the proportion of $u_1$ converted into $u_2$. The reaction terms are defined as:
\[
R_1(u_1, u_2) = -\alpha u_1, \quad R_2(u_1, u_2) = \beta \alpha u_1
\]

\subsection{Boundary Conditions}
Neumann boundary conditions with zero flux were applied along all edges of the spatial domain. To avoid boundary effects influencing the dynamics, the spatial domain was chosen to be significantly larger than the region of interest. Only the interior region was analyzed, where species concentrations remain unaffected by the boundaries.

% \subsection{$F_1$ Score Definition}
% \( F_1 \) Score, defined as follows:
% \[
% F_1 = \frac{2\text{TP}}{2\text{TP} + \text{FP} + \text{FN}}
% \]
% where TP, FP, TN, and FN are true positive count, false positive count, true negative count, and false negative count, respectively. Here, a positive is a graph edge that exists, and a negative is a graph edge that does not exist. To avoid divide-by-zero edge cases, we used an adapted version fully detailed by \citet{Nichol.2023.10.2172/1991387}.

\subsection{Estimation Procedure}
Mathematically, angle estimation can be described in the general case as:
\[
\hat{\theta} = \texttt{atan2}\left(\sum_{l \in \mathscr{P}(\texttt{C})} e_l \sin \theta_l, \sum_{l \in \mathscr{P}(\texttt{C})} e_l \cos \theta_l\right).
\]
where \texttt{atan2} is the signed arctangent function; $\mathscr{P}(\text{Parents}(C_1, C_2, \ldots, C_k)) = \mathscr{P}(\{d_i : d \in \{\texttt{NW}, \texttt{N}, \texttt{NE}, \texttt{E}, \texttt{SE}, \texttt{S}, \texttt{SW}, \texttt{W}\}, i \in \{1,2,\ldots,k\}\})$, for $i$ species, represents all potential parents of all center cells across all species; $e_l$ represents the strength of that edge (0 for non-present edges); and $\theta_l$ represents the angle of that edge ($135^{\circ}, 90^{\circ}, \dots, 180^{\circ}$).

\subsection{Experimental Details}
To evaluate the performance of \mcastle on realistic physical dynamics, we conducted a series of \ac{ADR} model experiments. Diffusion coefficients for the two species ($D_1$ and $D_2$) varied together between $0.005$ and $0.4$, advection velocities ($v$) ranged from $1.0$ to $3.0$, with angles ($\theta$) from $0^\circ$ to $90^\circ$. Reaction rates ($\alpha$) were set to $1.0$ with a scaling factor ($\beta$) of $1.0$, which capture the decay of $u_1$ into $u_2$. We explored ranges of $\alpha$ and $\beta$, but found they only scaled the modeled data and \mcastle's results were not affected. The initial concentration of $u_1$ was set to $50$ and $u_2$'s was set to zero. Varying initial concentration also had no effect on \mcastle results, so it is omitted. These experiments systematically explore the interplay between advection, diffusion, and reaction processes under diverse conditions to assess the method’s ability to infer causal dynamics.

\subsection{Hyperparameter Details}
For all \ac{ADR} experiments, we implemented \mcastle's \ac{PIP} with the \ac{PC}-stable~\citep{Colombo.2014} causal discovery algorithm. \mcastle inherits two hyperparameters from \ac{PC}-stable:
\begin{itemize}[leftmargin=*, itemsep=2pt]
    \item \textbf{P-value Threshold} ($PC_\alpha$): Used for individual independence tests.
    \item \textbf{Graph Threshold}: A p-value threshold used in the Benjamini–Hochberg false discovery rate correction~\citep{Benjamini.1995.10.1111/j.2517-6161.1995.tb02031.x}, applied after the initial graph estimation.
\end{itemize}
In the experiments presented, both thresholds were set to $0.01$. Results were found to be insensitive to these parameters.

\subsection{Causal Graph Diagram}
\Cref{fig:ADR_TSG} shows the ground truth causal graph for two species in our \ac{ADR} model. Both species persist from one time step to the next as they advect and diffuse in the space, denoted by the autodependency straight arrows, and $u_1$ decays into $u_2$, denoted by the curved arrow. The reaction graph from the multivariate stencil graph decomposition has the same form, enabling a one-to-one comparison.

\begin{figure}[htbp]
\begin{center}
\begin{tikzpicture}[
    ->, 
    >=stealth, 
    node distance=3cm and 4cm,
    thick,
    state/.style={
        draw, 
        circle, 
        minimum width=1.2cm,
        minimum height=1.2cm,
        font=\large,
        inner sep=0pt
    }
]
    % Time t nodes (left column)
    \node[state] (u1t) {$u_{1,t}$};
    \node[state, below=of u1t] (u2t) {$u_{2,t}$};
    
    % Time t+1 nodes (right column)
    \node[state, right=of u1t] (u1tp1) {$u_{1,t+1}$};
    \node[state, right=of u2t] (u2tp1) {$u_{2,t+1}$};
    
    % Auto-evolution edges (same species)
    \draw[->] (u1t) -- (u1tp1) node[midway, above, font=\small] {
        \begin{tabular}{c}
        diffusion/advection
        \end{tabular}
    };
    \draw[->] (u2t) -- (u2tp1) node[midway, below, font=\small] {
        \begin{tabular}{c}
        diffusion/advection
        \end{tabular}
    };
    
    % Cross-species interaction edges
    \draw[->] (u1t) to[bend left=20] node[midway, right, font=\small] {Decay} (u2tp1);
    
    % Axis labels
    \node[above=0.8cm of $(u1t)!0.5!(u1tp1)$, font=\large\bfseries] {Time};
    \node[left=0.7cm of $(u1t)!0.5!(u2t)$, font=\large\bfseries] {Species};
\end{tikzpicture}
\end{center}
\caption{Time series causal graph for the \ac{ADR} model. Nodes represent species \( u_1 \) and \( u_2 \) at each consecutive time step (\( t \) and \( t+1 \)). Straight arrows indicate intra-species relationships across time, capturing diffusion and advection dynamics. The curved arrow represents inter-species causal interactions governed by reaction terms \( R_1(u_1, u_2) \) and \( R_2(u_1, u_2) \).}
\label{fig:ADR_TSG}
\end{figure}
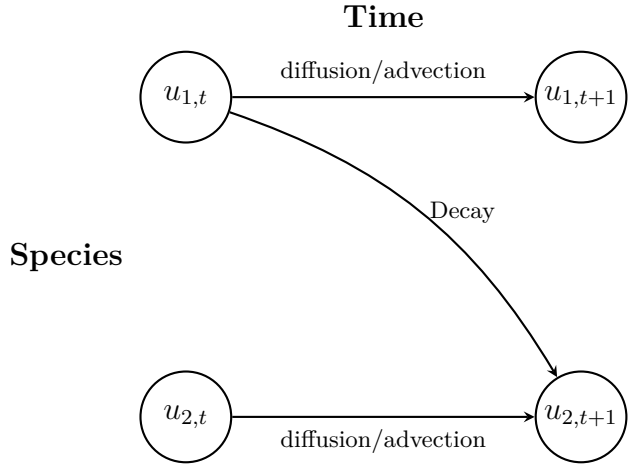

\section{Additional ADR Results}
\label{app:additional_adr_results}

\subsection{\( F_1 \) Score Distribution for Reaction Graph Estimations}
\label{app:adr_f1_scores}

\Cref{fig:reaction_graph_f1_appendix} provides a detailed histogram of the \( F_1 \) scores obtained from the reaction graph estimations in the \ac{ADR} model experiments (\( n = 672 \)). The distribution demonstrates that the majority of reaction graphs achieved perfect \( F_1 \) scores, with a median of \( 1.0 \) and a mean of \( 0.912 \) (red dashed line). A small subset of experiments resulted in \( F_1 \) scores below \( 0.8 \), and two experiments had \( F_1 = 0.0 \), indicating no links were identified in the graph. These results emphasize \mcastle's robustness in recoverable regimes and highlight the challenges in cases with weak or absent reaction signals.

\begin{figure}[ht]
    \centering
    \includegraphics[width=0.82\columnwidth]{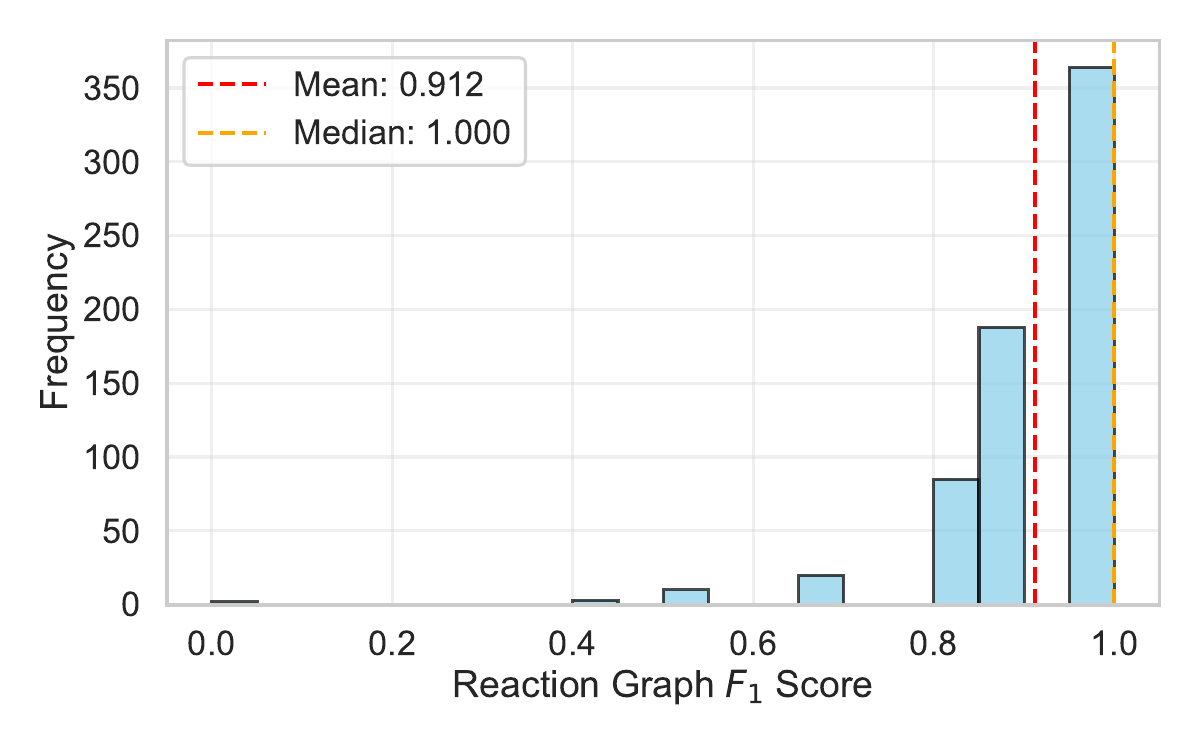}
    \caption{Histogram showing the frequency distribution of \( F_1 \) scores for reaction graph estimations (\( n = 672 \)). The distribution has a mean of \( 0.912 \) (red dashed line) and median of \( 1.000 \) (orange dashed line), indicating that the majority of reaction graphs achieved perfect \( F_1 \) scores, with a smaller subset showing lower performance. Rare cases with \( F_1 = 0 \) are when no links were identified in the graph.}
    \label{fig:reaction_graph_f1_appendix}
    % \Description[<short description>]{<long description>}
\end{figure}

\section{Supplementary Details on Atmospheric Chemistry and Experimental Setup}
\label{sec:appendix_atmospheric_chemistry}

\subsection{Background on Atmospheric Chemistry}
The Mount Pinatubo eruption in 1991 injected approximately 20 Tg of \ce{SO2} into the atmosphere~\citep{Guo.2004.10.1029/2003gc000654,Kremser.2016.10.1002/2015rg000511}, forming stratospheric sulfate aerosols that persisted for approximately two years. The resulting climate perturbation produced stratospheric warming of $\sim$1.5K and surface cooling of 0.2--0.5K~\citep{Dutton.1992.10.1029/92gl02495,parker1996,Soden.2002.10.1126/science.296.5568.727}, providing a natural analogue for stratospheric aerosol injection climate strategies~\citep{Trenberth.2007,Weylandt.2024.10.1175/JCLI-D-23-0267.1}. Ongoing research efforts focus on characterizing the detailed mechanisms of the Pinatubo response, particularly emphasizing the temporal evolution and geographic patterns of surface climate impacts, which provide crucial insights for climate intervention policy development~\citep{Weylandt.2024.10.1175/JCLI-D-23-0267.1}.

The \ac{E3SMv2-SPA} model used in this study incorporates atmosphere, land, oceanic, sea ice, land ice, and river components~\citep{Brown.2024.10.5194/gmd-17-5087-2024}. Particularly of note is the model's prognostic sulfate chemistry including \ce{SO2}, \ce{OH}, \ce{H2SO4}, and \ce{SO4} species. The intrinsic representation of aerosols within this modeling framework introduces additional complexity to the chemical pathway interpretation.

\subsection{Experimental Setup Details}
A large number of Earth system simulations were conducted for the CLimate Impact: Determining Etiology thRough pAthways (CLDERA) project at Sandia National Laboratories~\citep{Bull.2024.10.2172/2480139}. Experiments varied levels of Mt. Pinatubo prognostic aerosols in the \ac{E3SMv2-SPA} model to study with different levels of climate variability. Some sets of experiments used tagged aerosols from the eruption to track which aerosol particles originated from the eruption as they spread around the globe and underwent chemistry into different species.

During the early months of the Mt. Pinatubo eruption, stratospheric winds were consistently moving westward in the tropical latitudes, where Pinatubo resides~\citep{Thomas.2009.10.5194/acp-9-3001-2009}. We captured gridded data in the region between \SIrange{0}{30}{\degree}N and \SIrange{60}{90}{\degree}E, a small distance west of Mt. Pinatubo. The dataset has a $1^\circ$ spatial resolution (corresponding to approximately 107 km at 15 degrees N), with daily time samples per grid cell. We captured seven days of data after the eruption on June 15, 1991.

\subsection{Dataset Limitations}
The relatively coarse time sampling of daily data was problematic for estimating the dependence of FSDS without link assumptions. Additionally, all of the chemical species are column-integrated quantities, meaning that they are 2D fields per time step, representing the total mass or aerosol optical depth of the quantity through the entire atmosphere. This limitation is common in simulation datasets and observed satellite data.

\subsection{Algorithmic Details}
\label{app:chem_alg_details}
We implemented \mcastle's PIP using the \ac{PC}MCI algorithm~\citep{Runge.2019.10.1126/sciadv.aau4996}. We set $PC_\alpha$ and the graph-threshold hyperparameters to $0.01$. Finally, we applied a coefficient threshold of $0.125$ to the graph as a regularizer to remove some particularly weak dependencies. Reaction graph dependencies may have a strength below that threshold due to aggregation effects in the decomposition process.

\mcastle's link assumptions were applied to constrain the parent sets specifically for radiative flux. Similar to the link assumptions functionality in the Tigramite causal discovery Python package~\citep{tigramite}, these specified assumptions embed subject matter expertise into the causal discovery algorithm. In this case, we stipulated that FSDS cannot be a causal parent of any other species. We did not specify whether it could be a child of any species, nor did we constrain any other relationships. To be clear, chemical species estimation was not given any link assumptions.

\section{\ac{ENSO} Study's Algorithmic Details}
\label{app:enso_alg_details}
\ac{DYNOTEARS} has two core structure-learning hyperparameters for regularization~\citep{Pamfil.2020}:
\begin{itemize}[leftmargin=*, itemsep=2pt]
    \item $\lambda_a$ L1 regularization coefficient for contemporaneous dependencies and
    \item $\lambda_w$ L1 regularization coefficient for lagged dependencies.
\end{itemize}
In both cases, higher values create sparser graphs. There is also an edge-weight threshold, \texttt{w\_threshold}, below which edges are pruned to zero after optimization, and \texttt{max\_iter}, which sets the maximum number of iterations in the augmented Lagrangian method. 

Since this implementation of \mcastle only considers lag-1 relationships, $\lambda_a$ is unused.

The \ac{ENSO} case study in \Cref{sec:enso} used the following:
\begin{itemize}[leftmargin=*, itemsep=2pt]
    \item $\lambda_w = 0.01$ (default)
    \item \texttt{w\_threshold} $= 0.01$ (default $=0.0$)
    \item \texttt{max\_iter = 100} (default)
\end{itemize}

\end{document}

%% file: acronyms.tex
\DeclareAcronym{ENSO}{
    short=ENSO,
    long=El Ni\~{n}o Southern Oscillation
}

\DeclareAcronym{ERA5}{
    short=ERA5,
    long=European Centre for Medium-Range Weather Forecasts Reanalysis v5,
    cite=Hersbach.2020.10.1002/qj.3803
}

\DeclareAcronym{E3SM}{
    short=E3SM,
    long=Energy  Exascale  Earth  System Model,
    cite=e3sm-model
}

\DeclareAcronym{E3SMv2-SPA}{
    short=E3SMv2-SPA,
    long=Energy  Exascale  Earth  System Model version 2 with Stratospheric Prognostic Aerosols,
    cite=Brown.2024.10.5194/gmd-17-5087-2024
}

\DeclareAcronym{HSW-V}{
    short=HSW-V,
    long=Held-Suarez-Williamson-Volcanic,
    cite=Hollowed.2024.10.5194/gmd-17-5913-2024
}

\DeclareAcronym{ADR}{
    short=ADR,
    long=advection-diffusion-reaction
}

\DeclareAcronym{SCM}{
    short=SCM,
    long=structural causal model
}

\DeclareAcronym{CaStLe}{
    short=CaStLe,
    long=\textbf{Ca}usal \textbf{S}pace-Time S\textbf{t}encil \textbf{Le}arning
}

\DeclareAcronym{M-CaStLe}{
    short=M-CaStLe,
    long=Multivariate CaStLe
}

\DeclareAcronym{RFR}{
    short=RFR,
    long=random forest regression,
    cite=Breiman2001
}

\DeclareAcronym{ML}{
    short=ML,
    long=machine learning
}

\DeclareAcronym{ACC}{
    short=ACC,
    long=anomaly correlation coefficient
}

\DeclareAcronym{NSIDC}{
    short=NSIDC,
    long=National Snow \& Ice Data Center,
    cite=Peng2013
}

\DeclareAcronym{DOE}{
    short=DOE,
    long=United States Department of Energy,
}

\DeclareAcronym{MAE}{
    short=MAE,
    long=mean absolute error,
}

\DeclareAcronym{piControl}{
    short=piControl,
    long=pre-industrial control simulation,
}

\DeclareAcronym{CESM}{
    short=CESM,
    long=community Earth system model,
    cite=Kay2015
}

\DeclareAcronym{NOAA}{
    short=NOAA,
    long=National Oceanic and Atmospheric Administration,
    cite=NOAA_ERSST_V4
}

\DeclareAcronym{NCEP}{
    short=NCEP,
    long=National Centers for Environmental Prediction,
    cite=NCEP_Reanalysis
}

\DeclareAcronym{PIOMAS}{
    short=PIOMAS,
    long=Pan-Arctic Ice Ocean Modeling and Assimilation System,
    cite=Schweiger2011
}

\DeclareAcronym{CMIP}{
    short=CMIP,
    long=Coupled Model Intercomparison Project 
}

\DeclareAcronym{CMIP3}{
    short=CMIP3,
    long=phase 3 
}

\DeclareAcronym{CMIP5}{
    short=CMIP5,
    long=phase 5 
}

\DeclareAcronym{CMIP6}{
    short=CMIP6,
    long=phase 6,
    cite=Eyring2016
}

\DeclareAcronym{ESM}{
    short=ESM,
    long=Earth system model 
}

\DeclareAcronym{SVR}{
    short=SVR,
    long=support vector regression 
}

\DeclareAcronym{SIE}{
    short=SIE,
    long=sea ice extent
}

\DeclareAcronym{SIV}{
    short=SIV,
    long=sea ice volume
}

\DeclareAcronym{CLT}{
    short=CLT,
    long=total cloud cover percentage
}

\DeclareAcronym{FLWS}{
    short=FLWS,
    long=downward longwave flux at surface
}

\DeclareAcronym{PS}{
    short=PS,
    long=pressure at the surface
}

\DeclareAcronym{TS}{
    short=TS,
    long=temperature at the surface
}

\DeclareAcronym{SSH}{
    short=SSH,
    long=near-surface specific humidity
}

\DeclareAcronym{SST}{
    short=SST,
    long=sea surface temperature
}

\DeclareAcronym{OLR}{
    short=OLR,
    long=outgoing longwave radiation
}

\DeclareAcronym{uwind}{
    short=uwind,
    long=wind u component/zonal
}

\DeclareAcronym{vwind}{
    short=vwind,
    long=wind v component/meridional
}

\DeclareAcronym{AOD}{
    short=AOD,
    long=aerosol optical depth
}

\DeclareAcronym{EOF}{
    short=EOF,
    long=empirical orthogonal function
}

\DeclareAcronym{LENS}{
    short=LENS,
    long=Locally Encoded Neighborhood Structure
}

\DeclareAcronym{NDM}{
    short=NDM,
    long=neighborhood dependence matrix
}

\DeclareAcronym{ODE}{
    short=ODE,
    long=ordinary differential equation
}

\DeclareAcronym{PCA}{
    short=PCA,
    long=principal component analysis
}

\DeclareAcronym{PC}{
    short=PC,
    long=Peter-Clark
}

\DeclareAcronym{PCMCI}{
    short=PCMCI,
    long=PC Momentary Conditional Independence
}

\DeclareAcronym{DYNOTEARS}{
    short=DYNOTEARS,
    long=Dynamic-NOTEARS
}

\DeclareAcronym{PDE}{
    short=PDE,
    long=partial differential equation
}

\DeclareAcronym{PIP}{
    short=PIP,
    long=Parent-Identification Phase
}

\DeclareAcronym{PPR}{
    short=PPR,
    long=positive prediction rate
}

\DeclareAcronym{VAR}{
    short=VAR,
    long=vector autoregression model
}

\DeclareAcronym{RCT}{
    short=RCT,
    long=randomized control trial
}

\DeclareAcronym{DAG}{
    short=DAG,
    long=directed acyclic graph
}